\documentclass{article}

\usepackage{microtype}
\usepackage{graphicx}
\usepackage{subcaption}
\usepackage{booktabs} 

\usepackage{hyperref}

\usepackage[preprint]{icml2026}

\usepackage{algorithm}
\usepackage{algorithmic}

\usepackage{amsmath}
\usepackage{amssymb}
\usepackage{mathtools}
\usepackage{amsthm}

\usepackage[capitalize,noabbrev]{cleveref}

\theoremstyle{plain}
\newtheorem{theorem}{Theorem}[section]
\newtheorem{proposition}[theorem]{Proposition}
\newtheorem{lemma}[theorem]{Lemma}

\theoremstyle{definition}

\newtheorem{assumption}[theorem]{Assumption}
\theoremstyle{remark}

\newcommand{\E}{\mathbb{E}}
\newtheorem{claim}[theorem]{Claim}

\usepackage[textsize=tiny]{todonotes}

\icmltitlerunning{The Multi-Query Paradox in Zeroth-Order Optimization}

\begin{document}

\twocolumn[
  \icmltitle{The Multi-Query Paradox in Zeroth-Order Optimization}



  \icmlsetsymbol{equal}{*}

  \begin{icmlauthorlist}
    \icmlauthor{Wei Lin}{cuhkcse}
    \icmlauthor{Qingyu Song}{xmu}
    \icmlauthor{Hong Xu}{cuhkcse}
  \end{icmlauthorlist}

  \icmlaffiliation{cuhkcse}{Department of Computer Science and Engineering, The Chinese University of Hong Kong, Hong Kong}
  \icmlaffiliation{xmu}{Xiamen University, China}

  \icmlcorrespondingauthor{Hong Xu}{hongxu@cuhk.edu.hk}

  \icmlkeywords{Machine Learning, ICML}

  \vskip 0.3in
]



\printAffiliationsAndNotice{}  

\begin{abstract}


    Zeroth-order (ZO) optimization provides a powerful framework for problems where explicit gradients are unavailable and have to be approximated using only queries to function value. The prevalent single-query approach is simple, but suffers from high estimation variance, motivating a multi-query paradigm to improve estimation accuracy. This, however, creates a critical trade-off: under a fixed budget of queries (i.e. cost), queries per iteration and the total number of optimization iterations are inversely proportional to one another. How to best allocate this budget is a fundamental, under-explored question. 
    
    This work systematically resolves this query allocation problem. We analyze two aggregation methods: the de facto simple averaging (ZO-Avg), and a new Projection Alignment method (ZO-Align) we derive from local surrogate minimization. By deriving convergence rates for both methods that make the dependence on the number of queries explicit across strongly convex, convex, non-convex, and stochastic settings, we uncover a stark dichotomy: For ZO-Avg, we prove that using more than one query per iteration is always query-inefficient, rendering the single-query approach optimal. On the contrary, ZO-Align generally performs better with more queries per iteration, resulting in a full-subspace estimation as the optimal approach. Thus, our work clarifies that the multi-query problem boils down to a choice not about an intermediate query size, but between two classic algorithms, a choice dictated entirely by the aggregation method used. These theoretical findings are also consistently validated by extensive experiments.
\end{abstract}

\section{Introduction}
\label{sec:intro}

Zeroth-order (ZO) optimization provides a powerful framework for problems where explicit gradient information is unavailable or prohibitively expensive to compute \citep{conn2009introduction}. Unlike first-order methods that rely on analytical gradients, ZO algorithms guide their search using only function value queries to approximate gradients. Typically, it is achieved by evaluating the objective function along one or more random directions and observing the change in function value \citep{spall2002multivariate, mania2018simple}. This paradigm has recently gained significant attention due to its applicability to a wide array of challenging, high-dimensional problems, including fine-tuning of LLMs \citep{malladi2023fine, zhang2024revisiting}, policy optimization in reinforcement learning \citep{sun2020zeroth}, and hyperparameter tuning \citep{li2021zeroth}.

The predominant approach in modern ZO optimization is the single-query method, where a gradient estimate is constructed using one random direction per optimization step \citep{nesterov2017random}. While foundational, these estimators are inherently noisy, with an estimation variance that often scales with the problem's dimension \citep{nesterov2017random}. A natural extension to mitigate this issue is the multi-query approach, which employs a block of queries at each iteration to construct a more accurate gradient estimate \citep{duchi2015optimal}. 

However, this extension introduces a fundamental trade-off. 
Using multiple queries per iteration proportionally reduces the total number of iterations that can be performed under a fixed query budget. This tension gives rise to what we call the \emph{multi-query paradox}: increasing the number of queries per iteration improves the quality of each gradient estimate but simultaneously decreases the total number of updates one can afford.
A systematic investigation of this paradox requires addressing two basic questions: (1) What is the most effective method for combining information from multiple query directions into a single, coherent gradient estimate? (2) For a fixed total query budget, is it more efficient to perform many iterations with noisy, single-query estimates, or fewer iterations with higher-quality, multi-query estimates that consume more queries per iteration?

This work systematically addresses these questions. First, to resolve the aggregation problem, we propose a new multi-query estimator called Projection Alignment (ZO-Align), derived from first principles via local surrogate minimization. While similar estimators have been proposed from the perspective of sparse gradient recovery \citep{cai2022zeroth, wang2018stochastic}, our optimization-centric derivation provides a more fundamental justification for the method, independent of any sparsity assumptions. For a comprehensive comparison, we also analyze the canonical strategy of simply averaging estimates (ZO-Avg) \citep{duchi2015optimal}. Our initial analysis of the estimators' properties reveals that ZO-Align achieves a significantly lower mean squared error, motivating a deeper investigation into its performance within an optimization context. 

Second, to resolve the query allocation problem, we derive convergence rates for algorithms using both ZO-Align and ZO-Avg, making the dependence on the number of queries used per iteration explicit. We then leverage these rates to formally pose and solve an optimization problem: given a fixed total query budget, determine the allocation of queries per iteration that minimizes the final optimization error. To ensure the generality and robustness of our conclusions, this theoretical framework is systematically applied across a broad spectrum of problem structures, including strongly convex, convex, non-convex, and stochastic settings.

Our analysis reveals a stark dichotomy in optimal allocation strategies, critically dependent on the chosen aggregation method. For ZO-Avg, we prove that it cannot effectively utilize information from multiple queries: under a fixed budget, any attempt to use more than one query per iteration is detrimental, and the optimal strategy is always the single-query approach. In contrast, the ZO-Align estimator combines multi-query information in a geometry-aware way that avoids wasting queries on redundant directions, and in strongly convex and stochastic problems multi-query ZO-Align is more query-efficient than any single-query scheme. We further analyze the regime where the ambient dimension is much larger than the query block size, showing that ZO-Avg continues to follow the same single-query allocation rule, while ZO-Align degenerates into a step direction that is nearly colinear with a suitably rescaled ZO-Avg update due to the near-orthogonality of random directions in high dimensions.

To complement our theoretical analysis, we conducted extensive experiments across a suite of optimization problems. The empirical results consistently validate our theoretical findings, illustrating the practical implications of the query allocation dichotomy. Our work clarifies that the central design choice is not about tuning an intermediate query size, but rather about committing to an aggregation method that dictates whether the optimal algorithm is a sequential single-query method or a full-subspace estimation procedure.

The primary contributions of this work are as follows:
\begin{itemize}
    \item We propose the ZO-Align estimator based on local surrogate minimization, providing a clear optimization-centric perspective that complements existing geometric interpretations.
    \item We deliver the first comprehensive convergence analysis that directly compares ZO-Avg and ZO-Align across strongly convex, convex, non-convex, and stochastic settings, revealing a fundamental and consistent dichotomy in their performance.
    \item We provide rigorous theoretical proofs that under a fixed query budget, the optimal query allocation for ZO-Avg is always single-query, whereas ZO-Align can make better use of multi-query information, resolving the query allocation question for these two classes of estimators.
    \item We empirically validate our theoretical findings across a range of optimization problems, and also include high-dimensional tasks that illustrates how our query-allocation insights manifest in large-scale problems.

\end{itemize}


\section{Related Work}
\label{sec:related_work}

\paragraph{ZO gradient estimation.}
Zeroth-order optimization approximates gradients using finite-difference techniques, a concept explored through various random perturbation strategies. For instance, theoretical convergence has been demonstrated using Gaussian random noise perturbations \citep{nesterov2017random}, while other approaches have utilized uniform sampling \citep{flaxman2004online} and coordinate-wise perturbations \citep{lian2016comprehensive}. A distinct line of research focuses on problems where the gradient is assumed to be sparse, framing the estimation task as one of compressed sensing; here, techniques like LASSO \citep{wang2018stochastic} or other forms of regularized regression \citep{choromanski2020provably} enable the reconstruction of a sparse gradient from a limited number of function queries \citep{cai2022zeroth}. 

\paragraph{Query-efficient ZO optimizers.}
 A central challenge in zeroth-order optimization is its high query complexity, which has motivated a broad range of techniques aimed at improving query efficiency. Some methods construct more accurate local models for the update direction by incorporating second-order information \citep{ye2025hessian} or by exploiting problem geometry via Riemannian optimization \citep{he2024riemannian}. Other approaches focus on optimizing the querying process itself, for example by using compressed sensing to achieve double-logarithmic query complexity in sparse settings \citep{qiu2024gradient}, adaptively reusing past queries to reduce variance \citep{xiao2023lazy}, or learning query directions through reinforcement learning \citep{zhai2024learning}. Despite their differences, these approaches share a common focus: they seek to improve the quality of the gradient estimate at each individual iteration, typically under a fixed number of queries per step.

In contrast, our work takes a complementary perspective. Rather than designing yet another estimator for a given per-iteration query budget, we ask how to optimally allocate a fixed total query budget across iterations and aggregation schemes. Our analysis shows that, even for two canonical multi-query estimators, this allocation problem exhibits a sharp dichotomy: the optimal query-per-iteration choice collapses to opposite extremes depending on the aggregation method used.
\section{ZO Estimators: From Single-Query to Multi-Query}
\label{sec:block_zo}

This section introduces the core ZO gradient estimators used in our analysis. We begin with the standard single-query estimator before presenting two multi-query approaches: the intuitive Naive Averaging (ZO-Avg) and the principled Projection Alignment (ZO-Align), which we derive from local surrogate minimization. Then we analyze their Mean Squared Error (MSE) to establish a fundamental difference in their accuracy, which underpins the main convergence results in Section~\ref{sec:query_allocation}.

\subsection{Preliminaries and the Single-Query Estimator}

We consider the unconstrained optimization problem $\min_{x \in \mathbb{R}^d} f(x)$ where direct access to the gradient $\nabla f(x)$ is unavailable. Throughout this work, we assume the objective function $f$ is differentiable and $L$-smooth: 
\begin{equation*}
    \|\nabla f(x) - \nabla f(y)\| \le L \|x - y\|, \quad \forall x, y \in \mathbb{R}^d.
\end{equation*}
ZO algorithms navigate this problem by iteratively updating a parameter vector $x_t$ using a gradient estimate $\hat{g}(x_t)$ constructed from function value queries: $x_{t+1} = x_t - \eta_t \hat{g}(x_t)$.

The foundational building block for most ZO methods is the single-query estimator. It approximates the gradient by evaluating the function along a single random direction $u \in \mathbb{R}^d$ using a \emph{one-sided} finite-difference scheme. A common form is given by:
\begin{equation*}
    \hat{g}_u(x) = \frac{f(x + \mu u) - f(x)}{\mu} u,
\end{equation*}
where $\mu > 0$ is a small smoothing parameter. Throughout this work, we focus on an idealized regime in which $\mu\to 0$ and the estimator approximates the projection of the true gradient onto the direction $u$: 
\begin{equation*}
    \hat{g}_u(x) \overset{\mu\to 0}{\approx} (u^T \nabla f(x)) u.
\end{equation*}
All MSE and convergence results below are derived under this idealization. Importantly, the same geometric analysis applies to \emph{two-sided} finite-difference estimators, which differ from the one-sided form only in the higher-order $\mu$-dependent terms \cite{nesterov2017random}.

While fundamental, this single-direction estimate is inherently noisy, limiting its effectiveness. A natural extension is to probe the function along multiple directions to construct a more accurate gradient estimate.

\subsection{Simple Averaging Estimator (ZO-Avg)}
\label{sec:zo_avg}

The most direct method for aggregating information from a block of $q$ queries is to compute $q$ independent single-query estimates and average them. Let $U = [u_1, \dots, u_q] \in \mathbb{R}^{d \times q}$ be a matrix whose columns are random direction vectors drawn independently from $\mathcal{N}(0, I)$. We let $q \le d$ and assume the columns of $U$ are linearly independent. The averaging estimator, which we term ZO-Avg, is defined as:
\begin{equation}\label{eq:zo-avg}
    \begin{aligned}
    \hat{g}_{\text{AVG}}(x) &= \frac{1}{q} \sum_{i=1}^{q} \frac{f(x + \mu u_i) - f(x)}{\mu} u_i \\
    &\overset{\mu\to 0}{\approx}\left( \frac{1}{q} \sum_{i=1}^{q} u_i u_i^T \right) \nabla f(x) = \frac{1}{q} UU^T \nabla f(x).
\end{aligned}
\end{equation}
This estimator is appealing due to its simplicity and straightforward parallel implementation. 




\subsection{Projection Alignment Method (ZO-Align)}
\label{sec:zo_align}

Instead of simply averaging, we adopt a more principled, model-based perspective. The $q$ queries provide a snapshot of the function's local geometry within the column space of $U$. We seek the best update step of the form $-Uy$ for some coefficient vector $y \in \mathbb{R}^q$ by minimizing a local surrogate of the objective function. Given that $f$ is $L$-smooth, we can form a local quadratic upper bound around the point $x$:
\begin{equation*}
    f(x - Uy) \le f(x) - \nabla f(x)^T (Uy) + \frac{L}{2} \|Uy\|^2.
\end{equation*}
Minimizing this quadratic upper bound with respect to $y$ is a principled way to find the best local step within the column space. This yields the optimal coefficient vector:
\begin{equation*}
    y^* = \frac{1}{L} (U^T U)^{-1} U^T \nabla f(x).
\end{equation*} The optimal update direction is therefore proportional to $U y^* = U(U^T U)^{-1} U^T \nabla f(x)$. This motivates our second estimator, which replaces the true directional derivatives with their finite-difference approximations:
\begin{equation}\label{eq:zo-align}
    \hat{g}_{{\text{ALG}}}(x) = U(U^TU)^{-1} \begin{bmatrix} \frac{f(x+\mu u_1)-f(x)}{\mu} \\ \vdots \\ \frac{f(x+\mu u_q)-f(x)}{\mu} \end{bmatrix}.
\end{equation}
The name ``Projection Alignment'' comes from the powerful geometric property of this estimator. The matrix $P_U = U(U^T U)^{-1} U^T$ is the orthogonal projection operator onto the column space of $U$. Therefore, assuming $\mu\to 0$, $\hat{g}_{\text{ALG}}$ is the orthogonal projection of the true gradient onto the query subspace. This means that, by construction, its own projection onto each query direction $u_i$ is identical to the estimated projection of the true gradient: $u_i^T \hat{g}_{\text{ALG}} \approx u_i^T \nabla f(x)$.

This reveals a crucial difference from ZO-Avg. While ZO-Align is constructed to be geometrically consistent within the sampled subspace, ZO-Avg is not. The projection of the ZO-Avg estimate onto a query direction $u_i$ does not generally match the true directional derivative, leading to a less accurate representation of the local geometry. 




\subsection{Accuracy Analysis: A Tale of Two Estimators}
\label{sec:acc_analysis}

To formalize this comparison, we analyze the MSE of each estimator, defined as $\mathbb{E}[\|\hat{g}(x) - \nabla f(x)\|^2]$, which captures the total error from both bias and variance.

\begin{proposition}
\label{prop:mse}
Let $g = \nabla f(x)$ and assume the query directions are i.i.d. samples from $\mathcal{N}(0, I)$. The Mean Squared Error of the idealized estimators are:
\begin{enumerate}
    \item For ZO-Avg: $\quad \text{MSE}(\hat{g}_{\text{AVG}}) = \frac{d+1}{q} \|g\|^2$.
    \item For ZO-Align: $\quad \text{MSE}(\hat{g}_{\text{ALG}}) = \frac{d-q}{d} \|g\|^2$.
\end{enumerate}
\end{proposition}
\begin{proof}
    See Appendix~\ref{Proof:prop_mse}.
\end{proof}
For any $q > 1$, it is clear that $\frac{d-q}{d} < \frac{d+1}{q}$, indicating that ZO-Align offers a significantly lower estimation error. This fundamental difference in estimator quality is the primary driver of the performance gap we will establish in the following sections. It highlights that the principled, model-based approach of ZO-Align is far more effective at aggregating information from multiple queries than the naive averaging of ZO-Avg.


\section{Query Allocation Dichotomy}
\label{sec:query_allocation}


This section presents our main theoretical results, which establish a fundamental dichotomy in how multi-query ZO estimators should be utilized. 
Upon several standard settings, we analyze the convergence behavior of optimization algorithms using the ZO-Avg and ZO-Align estimators under a fixed total query budget. We first consider the deterministic setting, where we have exact function evaluations, for three classes of functions: $\mu$-strongly convex, convex, and non-convex. We then extend our analysis to the stochastic setting for convex functions, where function evaluations are corrupted by noise. Across all these settings, we show that the simple averaging of ZO-Avg is an inefficient mechanism for aggregating query information, making a fully sequential single-query method optimal. In direct contrast, the principled construction of ZO-Align effectively harnesses multi-query information, making the optimal strategy a full-subspace estimation.

\subsection{Query Allocation Problem}
Suppose we are given a fixed total query budget $K$. An algorithm performs a series of updates over $T$ iterations, and at each iteration $t$ it uses a set of $q_t$ function queries as stated in Section~\ref{sec:zo_avg}. The total number of queries is therefore constrained by $\sum_{t=0}^{T-1} q_t \le K$. Our central goal is to determine the optimal query allocation strategy---the sequence of queries per iteration, $\{q_t\}_{t=0}^{T-1}$---that minimizes the final optimization error.

\subsection{Optimal Query Allocation}
 
We begin by analyzing the case where function evaluations are deterministic. The update rule at iteration $t$ is $x_{t+1} = x_t - \eta_t \hat{g}(x_t)$, where $\hat{g}$ is either the ZO-Avg or ZO-Align estimator constructed with a set of $q_t$ query directions, as illustrated in Section~\ref{sec:block_zo}. Given an estimator with $q_t$ queries at iteration $t$, we choose the step size $\eta_t$ by minimizing the standard $L$-smooth upper bound on the expected decrease of $f$. The following proposition summarizes the resulting choices; these step sizes will be used throughout this section.

\begin{proposition}
\label{prop:optimal_stepsizes}
Let $f$ be an $L$-smooth function. Consider a single update of the form $x^{+} = x - \eta \hat{g}(x)$. In the idealized regime $\mu \to 0$, the following holds:

\textbf{(ZO-Avg}) If $\hat{g} = \hat{g}_{\text{AVG}}$ is the averaging estimator in \eqref{eq:zo-avg}, then the step size that maximizes the guaranteed decrease is
    \begin{equation}\label{eq:optimal_step_size_zoavg}
        \eta_{\text{AVG}}^*(q)
        = \frac{q}{L(q + d + 1)}.
    \end{equation}
(\textbf{ZO-Align}) If $\hat{g} = \hat{g}_{\text{ALG}}$ is the projection-alignment estimator in \eqref{eq:zo-align}, then step size that maximizes the guaranteed decrease in this bound is
    \begin{equation}\label{eq:optimal_step_size_zoalign}
        \eta_{\text{ALG}}^*(q)
        = \frac{1}{L}.
    \end{equation}
\end{proposition}
\begin{proof}
    See Appendix~\ref{Proof:prop_optimal_stepsizes}.
\end{proof}

\subsubsection{Strongly Convex Case}

A function is defined as $\gamma$-strongly convex if there exists a constant $\gamma$ such that for all $x$ and $y$ in its domain: 
\begin{equation*}
    f(y)\geq f(x)+\nabla f(x)^T(y-x) +\frac{\gamma}{2}\|y-x\|^2.
\end{equation*}
When the objective function $f(x)$ is both $L$-smooth and $\gamma$-strongly convex, we can achieve a linear convergence rate~\citep{boyd2004convex}. However, the rate's dependence on the number of queries per iteration, $q_t$, reveals the core dichotomy of our work. We first analyze the ZO-Avg estimator. 

\begin{theorem}
\label{thm:zo_avg_strong_convex}
Let $f$ be an $L$-smooth and $\gamma$-strongly convex function. The ZO-Avg algorithm with the optimal step size stated in \eqref{eq:optimal_step_size_zoavg} achieves the following convergence guarantee:
\begin{equation}
\begin{split}
    &\quad\mathbb{E}[f(x_{T}) - f(x^*)] \\ 
\le &\left( \prod_{t=0}^{T-1} \left(1 - \frac{\gamma q_t}{L(q_t+d+1)}\right) \right) (f(x_0) - f(x^*)),
\end{split}
\end{equation}
where $x^*$ represent the optimum. Under a fixed total query budget $K$, this convergence rate is optimized by choosing $q_t=1$ for all $t=0, \dots, K-1$, resulting in $T=K$ iterations and a rate of:
\begin{equation}
\mathbb{E}[f(x_K) - f(x^*)] \le \left(1 - \frac{\gamma}{L(d+2)}\right)^K (f(x_0) - f(x^*))
\end{equation}
\end{theorem}

\begin{proof}
    See Appendix~\ref{Proof:thm_zo_avg_strongly_convex}.
\end{proof}

The result in Theorem~\ref{thm:zo_avg_strong_convex} suggests that when using the simple averaging estimator, allocating more than one query per iteration is detrimental to the convergence rate. The core reason lies in the trade-off between variance reduction and the cost of queries. While increasing $q_t$ does reduce the variance of the gradient estimate, the improvement scales sublinearly $1/q_t$, as shown in Section~\ref{sec:acc_analysis}. However, the cost increases linearly with $q_t$. The analysis shows that the marginal benefit of reducing variance is not worth the cost of the additional queries. It is always better to spend the query budget on taking more, albeit noisier, steps. This exposes a fundamental flaw in the simple average strategy: it is an inefficient use of a limited query budget.

In stark contrast, the principled structure of the ZO-Align estimator leads to the opposite conclusion.


\begin{theorem}
\label{thm:zo_align_strong_convex}
Let $f$ be an $L$-smooth and $\gamma$-strongly convex function. The ZO-Align algorithm with the optimal step size $\eta_t = 1/L$ achieves the following convergence guarantee:
\begin{equation}
\mathbb{E}[f(x_{T}) - f(x^*)] \le \left( \prod_{t=0}^{T-1} \left(1 - \frac{\gamma q_t}{Ld}\right) \right) (f(x_0) - f(x^*)),
\end{equation}
where $x^*$ represent the optimum. Under a fixed total query budget $K$, this convergence rate is optimized by using as few iterations as possible with the maximum number of queries, i.e., $q_t=d$ for $T=K/d$ iterations. This yields the rate:
\begin{equation}
\mathbb{E}[f(x_{K/d}) - f(x^*)] \le \left(1 - \frac{\gamma}{L}\right)^{K/d} (f(x_0) - f(x^*)).
\end{equation}
\end{theorem}

\begin{proof}
    See Appendix~\ref{Proof:thm_zo_align_strongly_convex}.
\end{proof}

For ZO-Align, the optimal strategy is to use the maximum number of queries per step. The intuition is that ZO-Align does not simply average away noise; it constructs a principled, low-error projection of the gradient onto the query subspace. Optimization error is reduced multiplicatively at each step, meaning a more accurate gradient provides an exponential advantage over time. It is therefore more effective to spend the budget forming a single high-quality step than to take many uncertain steps. 

\subsubsection{Convex Case}
We then analyze the case where the objective function $f(x)$ is convex and $L$-smooth, but not necessarily strongly convex. We demonstrate that the optimal query allocation strategy is fundamentally different for the two estimators, reinforcing the core dichotomy of our work.

\begin{theorem}
\label{thm:zo_avg_convex}
Let $f$ be an $L$-smooth and convex function. The ZO-Avg algorithm with step size stated in \eqref{eq:optimal_step_size_zoavg} achieves the following convergence guarantee after $T$ iterations:
\begin{equation}
\begin{split}
     E[f(x_T) - f(x_*)]
     \le \frac{L(\|x_0 - x_*\|^2 + \frac{2}{L}(f(x_0) - f(x_*)))}{\sum_{t=0}^{T-1} \frac{2q_t}{q_t+d+1}}.
\end{split}
\end{equation}
Under a fixed total query budget $K$, this convergence rate is optimized by choosing $q_t=1$ for all $t$, resulting in $T=K$ iterations and a rate of:
\begin{equation}
\begin{split}
    &\quad E[f(x_K) - f(x_*)]\\
    &\le \frac{L(d+2)(\|x_0 - x_*\|^2 + \frac{2}{L}(f(x_0) - f(x_*)))}{2K}.
\end{split}
\end{equation}
\end{theorem}

\begin{proof}
    See Appendix~\ref{Proof:thm_zo_avg_convex}.
\end{proof}

Theorem~\ref{thm:zo_avg_convex} reinforces the conclusion from the strongly convex case: for the ZO-Avg estimator, using multiple queries per step is an inefficient use of the query budget. Next, we show that for the ZO-Align estimator, the conclusion remains the different.

\begin{theorem}
\label{thm:zo_align_convex}
Let $f$ be an $L$-smooth and convex function. The ZO-Align algorithm with step size $\eta_t=1/L$ achieves the following rate:
\begin{equation}
\begin{split}
    &\quad E[f(x_T) - f(x_*)] \\
    &\le \frac{d}{2\sum_{t=0}^{T-1} q_t} (L\|x_0 - x_*\|^2 + 2(f(x_0) - f(x_*)))
\end{split}
    \end{equation}
Under any allocation $\{q_t\}$ such that $\sum_{t=0}^{T-1} q_t=K$, ZO-Align yields the rate:
\begin{equation}
    E[f(x_T) - f(x_*)] \le \frac{d}{2K} (L\|x_0 - x_*\|^2 + 2(f(x_0) - f(x_*))) 
\end{equation}

\end{theorem}

\begin{proof}
    See Appendix~\ref{Proof:thm_zo_align_convex}.
\end{proof}

The ZO-Align rate depends only on the total number of queries $K$, not on how they are allocated across iterations. This key distinction stems directly from the mathematical properties of the estimators. The per-iteration progress for ZO-Align is linear in the number of queries. However for ZO-Avg, it is a nonlinear concave function. 

\subsubsection{Non-Convex Case}
Then we proceed to non-convex setting, where the objective is only assumed to be $L$-smooth. In this scenario, the goal is typically to find a first-order stationary point~\citep{nocedal2006numerical}, i.e., a point $x$ where $\|\nabla f(x)\|^2 \le \epsilon$. The convergence is measured by the number of queries required to achieve this goal. The analysis reveals the same fundamental dichotomy between the two estimators.

\begin{theorem}
\label{thm:zo_avg_nonconvex}
Let $f$ be an $L$-smooth function. The ZO-Avg algorithm with step size $\eta_t = \frac{q_t}{L(q_t+d+1)}$ ensures that:
\begin{equation}
\min_{t=0,...,T-1}\mathbb{E}[\|\nabla f(x_t)\|^2] \le \frac{2L(f(x_0)-f^*)}{\sum_{t=0}^{T-1} \frac{q_t}{q_t+d+1}},
\end{equation}
where $f^*$ is the minimum value of $f$. Under a fixed total query budget $K$, this rate is optimized by choosing $q_t=1$ for all $t$, which gives $T=K$ iterations and yields:
\begin{equation}
\min_{t=0,...,K-1}\mathbb{E}[\|\nabla f(x_t)\|^2] \le \frac{2L(d+2)(f(x_0)-f^*)}{K} 
\end{equation}
\end{theorem}

\begin{proof}
    See Appendix~\ref{Proof:thm_zo_avg_nonconvex}.
\end{proof}

\begin{theorem}
\label{thm:zo_align_nonconvex}
Let $f$ be an $L$-smooth function. The ZO-Align algorithm with step size $\eta_t=1/L$ ensures that:
\begin{equation}
\min_{t\in\{0,,1...,T-1\}}\mathbb{E}[\|\nabla f(x_t)\|^2] \le \frac{2Ld(f(x_0)-f^*)}{\sum_{t=0}^{T-1}q_t}
\end{equation}
Under a fixed total query budget $K$, any allocation $\{q_t\}$ that exhausts the budget, $\sum_{t=0}^{T-1} q_t = K$, yields the rate:
\begin{equation}
\min_{t\in\{0,,1...,T-1\}}\mathbb{E}[\|\nabla f(x_t)\|^2] \le \frac{2Ld(f(x_0)-f^*)}{K} 
\end{equation}
\end{theorem}

\begin{proof}
    See Appendix~\ref{Proof:thm_zo_align_nonconvex}.
\end{proof}

The conclusions for the non-convex setting mirror those from the convex case, reinforcing the conclusion of multi-query paradox. While the same fundamental difference in optimal allocation holds, the performance gap between the two methods is less pronounced here than in the strongly convex setting, boiling down to a constant factor difference rather than an exponential one.

\paragraph{Stochastic Case} We now extend our analysis to the stochastic setting, a standard framework in modern machine learning where the objective is an expectation: $f(x) = \E_{\xi}[F(x, \xi)]$. We show that the fundamental dichotomy between ZO-Avg and ZO-Align persists.

For this analysis, we adopt standard assumptions: $f(x)$ is convex, each component function $F(x, \xi)$ is L-smooth, and the stochastic gradients have bounded variance at the optimum, $x^* = \arg\min_x f(x)$.

\begin{assumption}
\label{asm:bounded_variance}
The variance of the stochastic gradient at the optimum is bounded, i.e., 
$$
\E_\xi[\|\nabla F(x^*, \xi) - \nabla f(x^*)\|^2] = \E_\xi[\|\nabla F(x^*, \xi)\|^2] \le \sigma^2.
$$
\end{assumption}

This leads to the following standard lemma which bounds the expected squared norm of the stochastic gradient.

\begin{lemma}
\label{lem:sto_lsmooth}
For a convex and L-smooth function $f(x)$ satisfying Assumption~\ref{asm:bounded_variance}, we have $\E_\xi[\|\nabla F(x, \xi)\|^2] \le 4L(f(x)-f(x^*)) + 2\sigma^2$.
\end{lemma}

\begin{proof}
    See Appendix~\ref{Proof:lem_sto_lsmooth}.
\end{proof}

This lemma provides the key tool for controlling the error introduced by the stochastic oracle, which we now use to derive the convergence guarantee for two estimators.


\begin{theorem}
\label{thm:zo_avg_stochastic}
Let $f$ be an L-smooth and convex function satisfying Assumption 4.8. The ZO-Avg algorithm using a diminishing step size $\eta_t = \frac{\eta_0}{\sqrt{t+1}}$ with $\eta_0 \le \frac{1}{4L(d+2)}$ satisfies the stability condition $\eta_t \le \frac{q_t}{4L(q_t+d+1)}$ for all $t \ge 0$ . After $T$ iterations, it achieves the following convergence guarantee:
\begin{equation}
    \mathbb{E}[f(\overline{x}_T) - f(x^*)] \le \frac{||x_0 - x^*||^2 + 2\eta_0^2\sigma^2 \sum_{t=0}^{T-1} \frac{q_t+d+1}{q_t(t+1)}}{\eta_0 \sum_{t=0}^{T-1} \frac{1}{\sqrt{t+1}}}
\end{equation}
where $\overline{x}_T$ is the weighted average of the iterates. Under a fixed total query budget $K$, this convergence rate is optimized by choosing $q_t = 1$ for all $t$.
\end{theorem}

\begin{proof}
    See Appendix~\ref{Proof:thm_zo_avg_sto}.
\end{proof}

\begin{theorem}
\label{thm:zo_align_stochastic}
Let $f$ be an L-smooth and convex function satisfying Assumption~\ref{asm:bounded_variance}. The ZO-Align algorithm using a diminishing step size $\eta_t = \frac{\eta_0}{\sqrt{t+1}}$, with $\eta_t \le \frac{1}{4L}$, achieves the following convergence guarantee after $T$ iterations:
\begin{equation*}
\E[f(\bar{x}_T) - f(x^*)] \le \frac{d\|x_0-x^*\|^2 + 2\eta_0^2\sigma^2 \sum_{t=0}^{T-1} \frac{q_t}{t+1}}{\eta_0 \sum_{t=0}^{T-1} \frac{q_t}{\sqrt{t+1}}}
\end{equation*}
where $\bar{x}_T$ is the weighted average of the iterates. Under a fixed total query budget $K$, this convergence rate is optimized by choosing $q_t=d$ for $T=K/d$ iterations.
\end{theorem}

\begin{proof}
    See Appendix~\ref{Proof:thm_zo_align_sto}.
\end{proof}

\subsection{The Multi-Query Paradox}
\label{sec:multi_query_paradox}

Taken together, the theoretical results reveal a simple dichotomy under a fixed total query budget. For ZO-Avg, all four regimes we study—strongly convex, convex, non-convex, and stochastic convex—lead to the same conclusion: the optimal allocation is to take a single query per iteration, so that the budget is spent on as many steps as possible. For ZO-Align, the picture is the opposite in the strongly convex and stochastic convex cases, where the best strategy is to use as many queries as possible per step (up to $d$ when feasible), while in the deterministic convex and non-convex cases the convergence rate depends only on the total number of queries and is essentially indifferent to how they are distributed across iterations.

This dichotomy has a clear geometric interpretation. ZO-Avg forms its update by averaging directional derivatives along a block of sampled directions, but it never accounts for the geometry of these directions. If several query vectors happen to be close to each other, their information is effectively counted multiple times, while directions that would explore new parts of the space may be underrated. Additional queries therefore provide diminishing returns: they mainly reduce variance through repeated measurements along nearly redundant directions, and this variance reduction scales sublinearly in the block size. Once the linear cost in queries is taken into account, it is always better for ZO-Avg to take more, cheaper steps with a single direction per step than to aggregate many directions at once.

By contrast, ZO-Align uses the entire block of directional measurements to reconstruct the orthogonal projection of the gradient onto the span of the sampled directions. The Gram matrix is explicitly inverted, so correlations and unequal norms among the queries are compensated, and the resulting update is the unique vector in the sampled subspace whose projections match the observed directional derivatives. As a consequence, each additional query genuinely enlarges and refines the subspace in which the method can make informed progress, rather than merely re-averaging noisy measurements along similar directions. This difference in how query information is used explains why ZO-Align benefits from larger blocks in the strongly convex case, where every improvement in the quality of the update translates into a better contraction factor and hence compounds multiplicatively over iterations. In the limit where the block size equals the dimension and the query matrix has full rank, ZO-Align recovers the true gradient and reduces to a classical full-subspace finite-difference method.


\begin{figure*}[!htb]
    \centering
    \begin{minipage}{0.49\textwidth}
        \includegraphics[width=\linewidth]{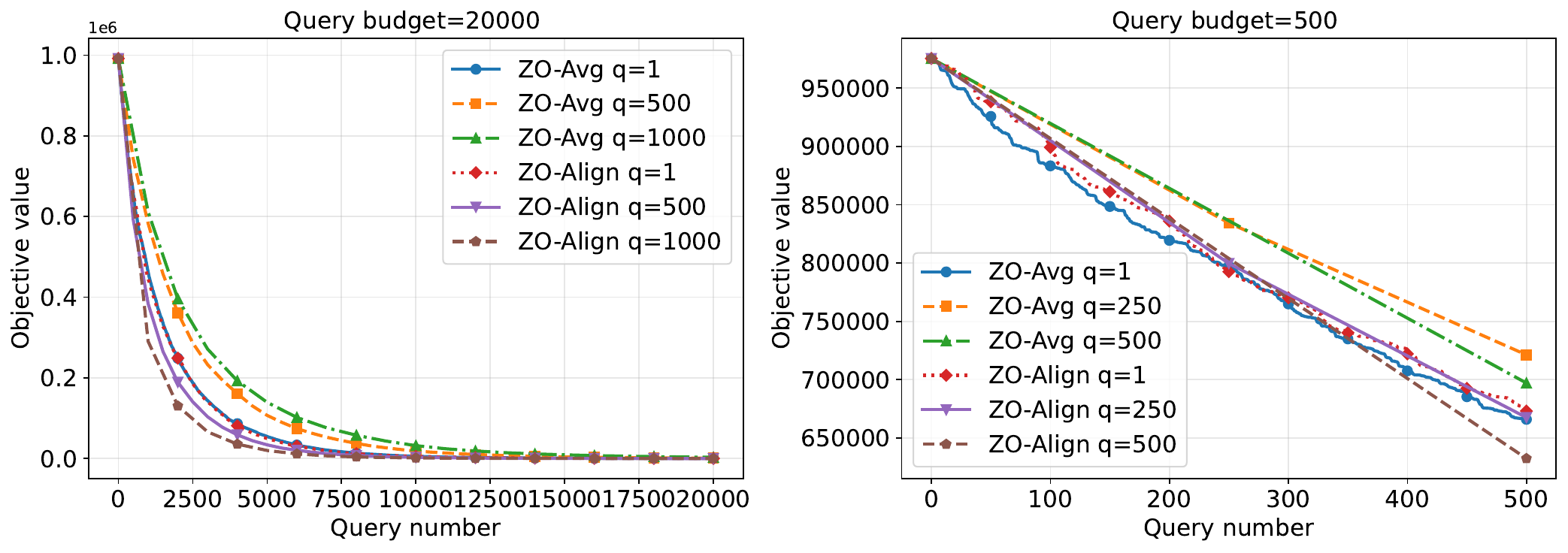}
        \caption{Objective with query used: Strongly convex case.}
        \label{fig:strongly_convex}
    \end{minipage}
    \hfill
    \begin{minipage}{0.49\textwidth}
        \includegraphics[width=\linewidth]{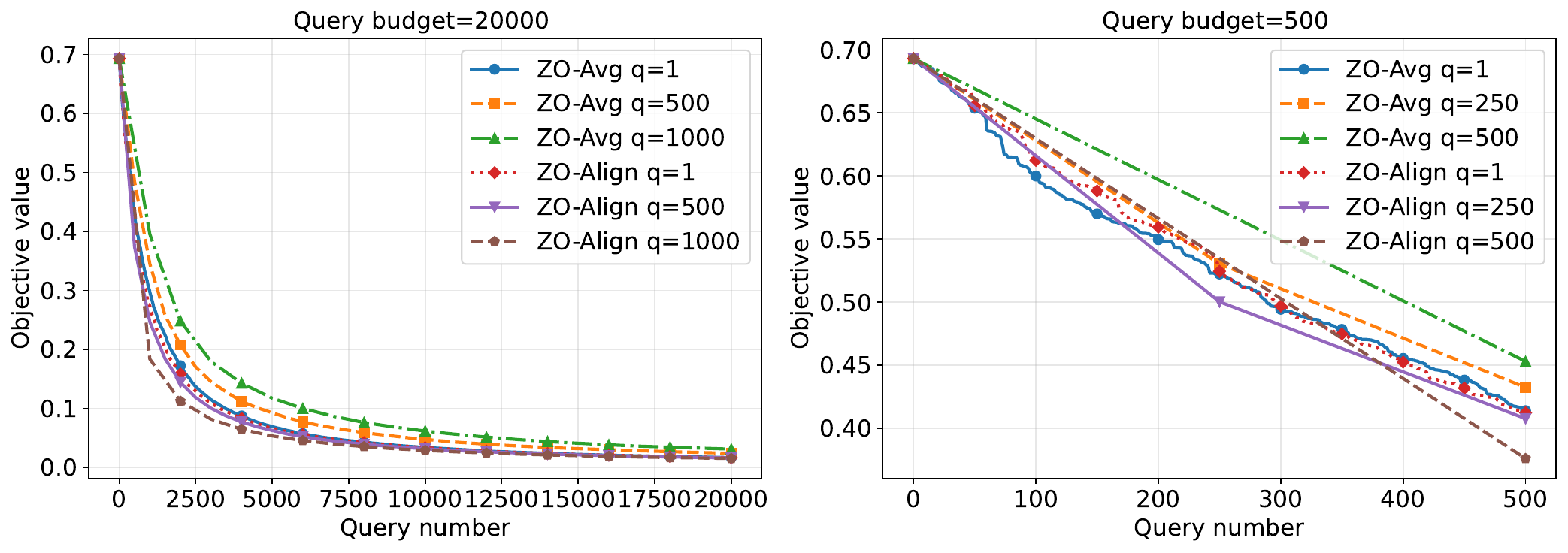}
        \caption{Objective with query used: Convex case.}
        \label{fig:convex}
    \end{minipage}
    \vspace{-3mm}
\end{figure*}

\begin{figure*}[!htb]
    \centering
    \begin{minipage}{0.49\textwidth}
        \includegraphics[width=\linewidth]{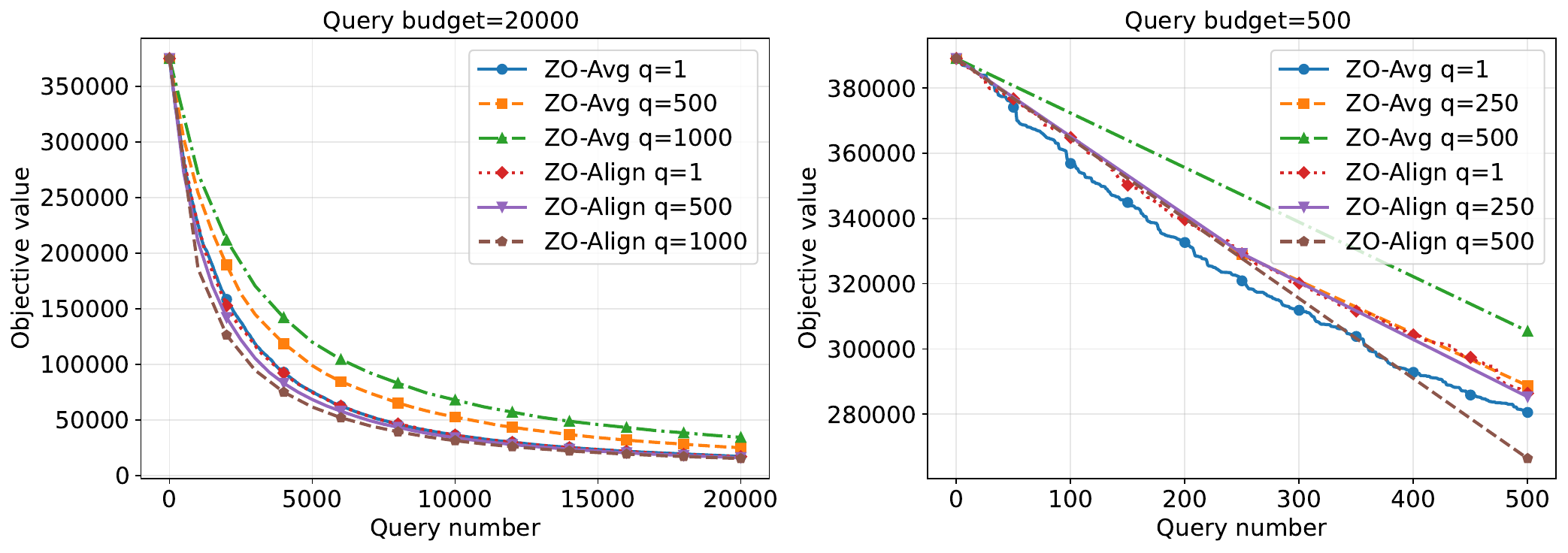}
        \caption{Objective with query used: Nonconvex case.}
        \label{fig:nonconvex}
    \end{minipage}
    \hfill
    \begin{minipage}{0.49\textwidth}
        \includegraphics[width=\linewidth]{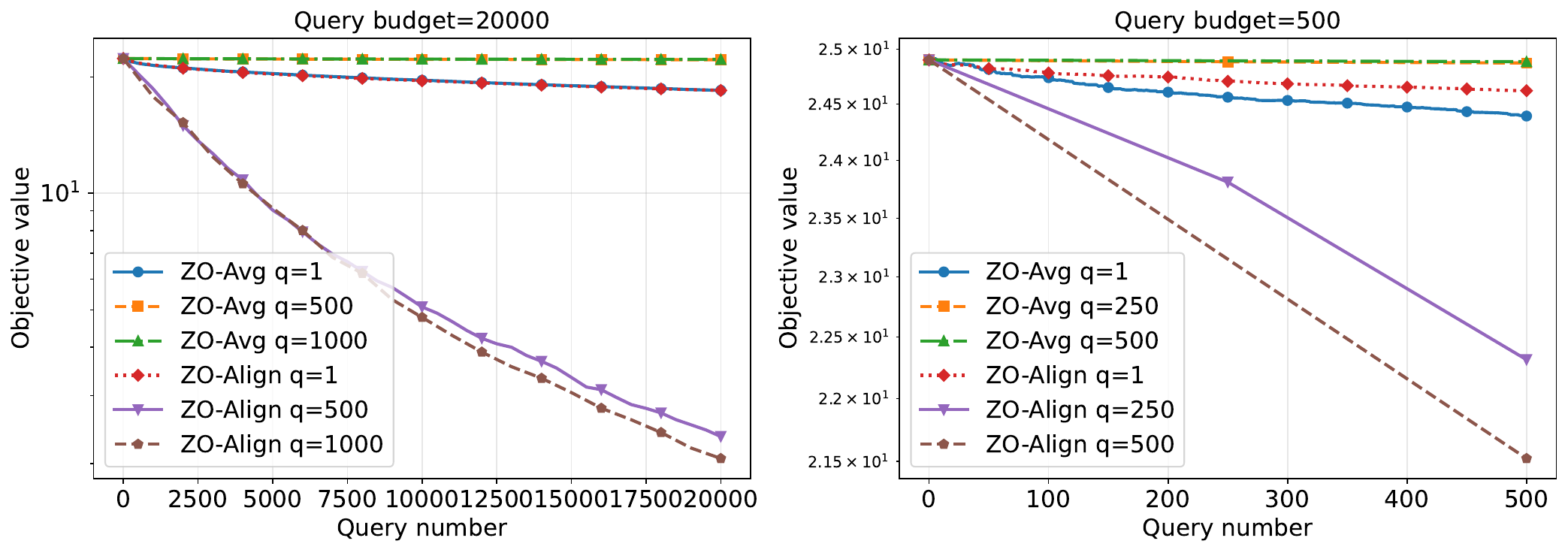}
        \caption{Objective with query used: Stochastic case.}
        \label{fig:sto_case}
    \end{minipage}
    \vspace{-5mm}
\end{figure*}
\section{High-Dimensional Regime}
\label{sec:high_dim_approx}

Modern applications such as large-scale neural network fine-tuning often operate in a regime where the ambient dimension $d$ is extremely large, whereas the query batch size $q$ is constrained to be much smaller than $d$. In this setting, when the columns of $U \in \mathbb{R}^{d\times q}$ are i.i.d.\ samples from $\mathcal{N}(0,I)$, the Gram matrix $U^\top U$ concentrates sharply around a scaled identity:
\begin{equation*}
    U^\top U \;\approx\; d\, I_q \qquad \text{when } d \gg q,
\end{equation*}
so that the sampled directions are nearly mutually orthogonal and $(U^\top U)^{-1}$ is, to leading order, a scalar multiple of the identity. A rigorous analysis and empirical verification of the approximation are provided in Appendix~\ref{app:high_dim}.

This observation clarifies how the benefits of ZO-Align evolve in high dimensions. The main advantage of ZO-Align over simple averaging is that it explicitly accounts for correlations among the sampled directions through the factor $(U^\top U)^{-1}$. When $d$ is moderate and $q$ is a nontrivial fraction of $d$, the Gram matrix $U^\top U$ can deviate substantially from a scaled identity, and this reweighting meaningfully corrects for the non-orthogonality and unequal norms of the sampled directions. In contrast, when $d \gg q$ and $U^\top U \approx d I_q$, the sampled directions are already almost orthogonal and have comparable lengths, so the correction degenerates to an almost uniform scaling. As a result, the geometric distinction between ZO-Align and ZO-Avg within the sampled subspace becomes much less pronounced in this extreme high-dimensional regime. 

\section{Experiments}
\label{sec:exp}

To complement our theoretical analysis, we evaluate ZO-Avg and ZO-Align under fixed query budgets on both classical objectives and a LLM fine-tuning task. Throughout, we plot performance against the \emph{total number of queries}, so that the horizontal axis directly reflects zeroth-order computational cost.

\subsection{Classical Optimization Problems}

We first consider four standard objectives at dimension $d = 1000$, representing the regimes studied in Section~\ref{sec:query_allocation}: a strongly convex quadratic, convex logistic regression on synthetic data, the non-convex Rosenbrock function, and a stochastic convex logistic regression objective (details in Appendix~\ref{sec:appendix_exp_setup}). For each problem we run ZO-Avg and ZO-Align with several query sizes $q$ per iteration under two total budgets, $K = 20{,}000$ and $K = 500$ (the latter precluding a single full-subspace step). Each configuration is averaged over multiple random seeds.

Across all four problems, the empirical behaviour closely matches our theoretical allocation results. For ZO-Avg, the best performance under a fixed budget is consistently obtained with $q=1$, and increasing $q$ monotonically harms query efficiency, in line with our Theorems. For ZO-Align, larger query blocks yield clear gains in the strongly convex and stochastic convex cases, where the convergence factors depend linearly on $q$, while in the deterministic convex and non-convex problems the curves for different $q$ are much closer, reflecting the allocation-insensitive rates. Overall, these experiments confirm the query allocation dichotomy of Section~\ref{sec:multi_query_paradox}: ZO-Avg prefers single-query steps, whereas ZO-Align can exploit larger blocks without the collapse in efficiency seen for ZO-Avg.

\begin{figure}
    \centering
    \includegraphics[width=\linewidth]{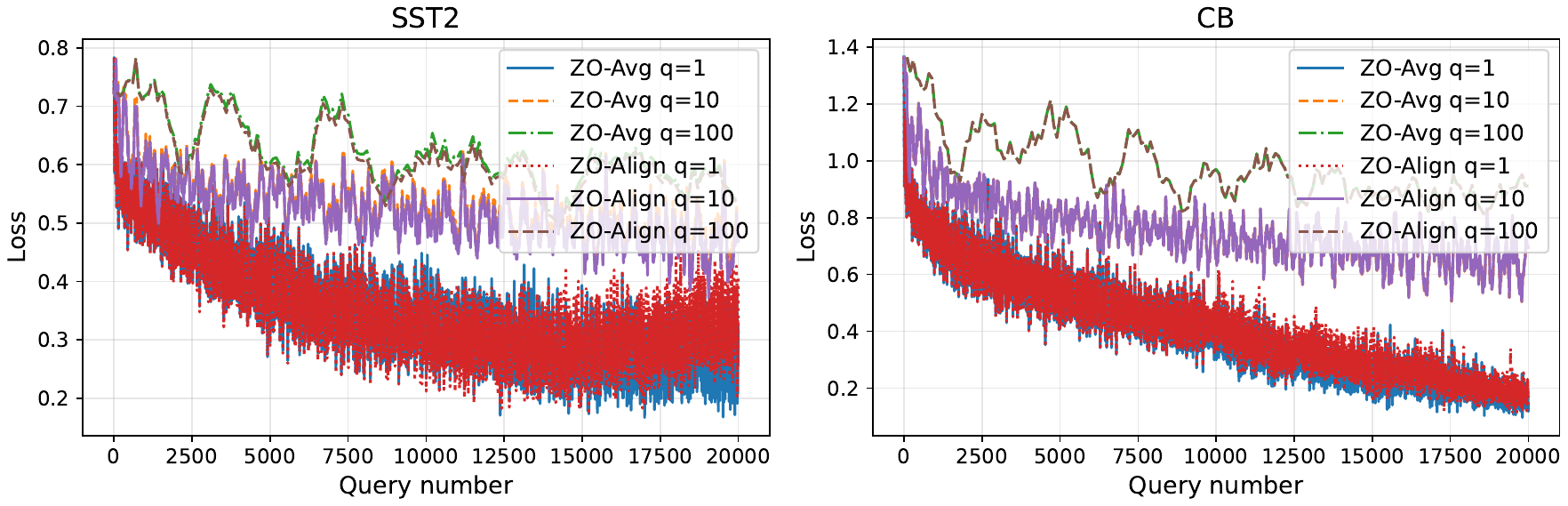}
    \caption{Finetuning Qwen3-0.6b on SST2 and CB. }
    \label{fig:llm_ft}
    \vspace{-5mm}
\end{figure}

\subsection{High-Dimensional Language Model Fine-Tuning}
\label{sec:exp_high_dim}

We next examine a high-dimensional setting where the parameter dimension $d$ is extremely large while the feasible query batch size $q$ remains modest ($q \ll d$). We fine-tune the Qwen3-0.6B model on two sentence classification benchmarks, SST-2 and CB, using zeroth-order optimization on the cross-entropy loss. The parameters being updated are treated as a single vector in $\mathbb{R}^d$; further protocol details are given in Appendix~\ref{sec:appendix_exp_setup}.

We compare ZO-Avg with an approximate variant of ZO-Align adapted to the high-dimensional regime. For ZO-Align, forming and inverting the full Gram matrix $U^\top U$ is prohibitively expensive at this scale. Motivated by the concentration result in Section~\ref{sec:high_dim_approx}, we approximate the Gram matrix by its diagonal. We fix a total query budget of $K = 20{,}000$ function evaluations and vary the per-iteration query count over
$q \in \{1,10,100\}$. The step size is tuned separately for each (estimator, $q$) pair by grid search.

As shown on Figure~\ref{fig:llm_ft}, on both SST-2 and CB, ZO-Avg continues to favour small query blocks: the best optimization performance per query is achieved for $q=1$, and increasing $q$ consistently slows progress under the fixed budget. The diagonal ZO-Align variant closely tracks ZO-Avg across all choices of $q$, with the two sets of curves almost indistinguishable at the scale of our plots. This behaviour matches the high-dimensional analysis in Section~\ref{sec:high_dim_approx}: when $d \gg q$, the Gram matrix concentrates near a scaled identity, so the projection step in ZO-Align reduces essentially to a rescaled version of the ZO-Avg update. In realistic large-scale language-model fine-tuning, the query allocation dichotomy for ZO-Avg therefore remains clearly visible, while the additional geometric structure of ZO-Align provides little practical gain beyond what can be achieved by a suitably tuned ZO-Avg scheme.


\section{Conclusions}
\label{sec:conclusions}
This work addresses two core questions in zeroth-order optimization: how to aggregate information from multiple queries and how to allocate a fixed query budget. By analyzing simple averaging ZO-Avg and the projection-based ZO-Align estimator across strongly convex, convex, non-convex, and stochastic settings, we show that aggregation and allocation are tightly coupled and lead to a sharp dichotomy. Under a fixed budget, the optimal strategy for ZO-Avg is always to use single-query steps, whereas ZO-Align can benefit from larger query blocks. Experiments validate the theoretical analysis. Our results resolve the multi-query paradox by clarifying that the central design choice is not to tune an intermediate query size, but to select an aggregation rule that determines whether the canonical algorithm is a sequential single-query scheme or a projection-based multi-query method.


\section*{Impact Statement}
This paper presents work whose goal is to advance the field of Machine Learning. There are many potential societal consequences of our work, none which we feel must be specifically highlighted here.




\bibliography{icml2026}
\bibliographystyle{icml2026}


\newpage
\appendix
\onecolumn

\clearpage

\onecolumn
\appendix

\begin{center}
    {\Large \bfseries Appendix}   
\end{center}

\section{High-Dimensional Behaviour of ZO-Align}
\label{app:high_dim}

In this appendix we analyze the behaviour of ZO-Align in the regime where the dimension $d$ is very large while the query batch size $q$ remains much smaller than $d$. Throughout, we retain the notation from the main text: $U \in \mathbb{R}^{d\times q}$ collects the $q$ query directions as columns, each sampled i.i.d.\ from $\mathcal{N}(0,I_d)$, and
\[
P_U \;=\; U (U^\top U)^{-1} U^\top
\]
denotes the orthogonal projection onto the column space of $U$. In the idealized limit $\mu\to 0$, ZO-Align can be written as
\[
\hat{g}_{ALG}(x) \;=\; P_U \nabla f(x),
\]
so understanding the high-dimensional behaviour of ZO-Align reduces to understanding the random projector $P_U$.

\subsection{Concentration of the Gram Matrix}
\label{app:high_dim_gram}

It is convenient to view $U$ row-wise. Let $r_1,\dots,r_d \in \mathbb{R}^q$ denote the rows of $U$. Since the columns of $U$ are i.i.d.\ $\mathcal{N}(0,I_d)$, each row $r_k$ is an independent sample from $\mathcal{N}(0,I_q)$ and
\[
U^\top U
= \sum_{k=1}^d r_k^\top r_k, 
\qquad
\frac{1}{d} U^\top U
= \frac{1}{d} \sum_{k=1}^d r_k^\top r_k
\]
is the empirical covariance matrix of $d$ i.i.d.\ Gaussian samples in $\mathbb{R}^q$.

The following standard result quantifies the concentration of this empirical covariance around the identity.

\begin{lemma}
\label{lem:gram_concentration}
There exist absolute constants $c,C>0$ such that the following holds. Let $U \in \mathbb{R}^{d\times q}$ have i.i.d.\ $\mathcal{N}(0,1)$ entries. Then for all $t \in (0,1)$,
\[
\Pr\!\left(
  \left\| \frac{1}{d} U^\top U - I_q \right\|_{op}
  \;\ge\; C\Big( \sqrt{\frac{q}{d}} + t \Big)
\right)
\;\le\; 2 \exp(-c d t^2).
\]
In particular, if $d \gg q$ then with probability at least $1 - 2 \exp(-c' q)$,
\[
\left\| \frac{1}{d} U^\top U - I_q \right\|_{op}
= O\!\left(\sqrt{\frac{q}{d}}\right).
\]
\end{lemma}

\begin{proof}
This is a standard consequence of sample covariance concentration for isotropic sub-Gaussian vectors; see, e.g., the general treatments in Chapter 5 of \cite{vershynin2018high}. For completeness, we specialize the argument to the Gaussian case.

Row representation and reduction to one-dimensional marginals. Write $U$ row-wise as
\[
U = \begin{bmatrix} r_1^\top \\ \vdots \\ r_d^\top \end{bmatrix},
\qquad
r_k \in \mathbb{R}^q,\quad r_k \sim \mathcal{N}(0,I_q)\ \text{i.i.d.}
\]
Then
\[
\frac{1}{d} U^\top U
= \frac{1}{d} \sum_{k=1}^d r_k r_k^\top
\]
is the empirical covariance matrix of $d$ i.i.d.\ Gaussian samples in $\mathbb{R}^q$.

Let
\[
\Sigma_d := \frac{1}{d} U^\top U,
\qquad
A := \Sigma_d - I_q.
\]
We need to control $\|A\|_{op}$. For any $x \in \mathbb{R}^q$ with $\|x\|=1$, we have
\[
x^\top A x
= x^\top \Sigma_d x - x^\top x
= \frac{1}{d} \sum_{k=1}^d \langle r_k, x \rangle^2 - 1.
\]
Define
\[
Z_k(x) := \langle r_k, x \rangle,
\qquad
Y_k(x) := Z_k(x)^2 - 1.
\]
For fixed $x$ with $\|x\|=1$, $Z_k(x) \sim \mathcal{N}(0,1)$ i.i.d., so
\[
\sum_{k=1}^d Z_k(x)^2 \sim \chi^2_d.
\]
Consequently, $\frac{1}{d} \sum_{k=1}^d Z_k(x)^2$ concentrates around $1$, and we can use a sharp tail inequality for the chi-square distribution due to Laurent and Massart~\cite{laurent2000adaptive}.

Let $S(x) := \sum_{k=1}^d Z_k(x)^2$. By definition,
\[
x^\top A x
= \frac{1}{d} S(x) - 1.
\]
Eq. 4.3-4.4 in \cite{laurent2000adaptive} show that if $S \sim \chi^2_d$, then for all $u>0$,
\begin{align*}
\Pr\!\big(S - d \ge 2\sqrt{d u} + 2u \big) &\le e^{-u}, \\
\Pr\!\big(d - S \ge 2\sqrt{d u} \big) &\le e^{-u}.
\end{align*}
Translating these inequalities to $\frac{1}{d} S(x) - 1$ yields: for all $u>0$,
\begin{align*}
\Pr\!\left( \frac{1}{d} S(x) - 1 \ge 2 \sqrt{\frac{u}{d}} + 2 \frac{u}{d} \right)
&\le e^{-u}, \\
\Pr\!\left( 1 - \frac{1}{d} S(x) \ge 2 \sqrt{\frac{u}{d}} \right)
&\le e^{-u}.
\end{align*}
Taking $u \in (0,d)$, we have $\frac{u}{d} \le 1$, so there exists an absolute constant $C_0>2$ such that
\[
2 \sqrt{\frac{u}{d}} + 2 \frac{u}{d}
\le C_0 \left( \sqrt{\frac{u}{d}} \right).
\]
Hence, for all $u \in (0,d)$,
\[
\Pr\!\left( \left| \frac{1}{d} S(x) - 1 \right|
          \ge C_0 \sqrt{\frac{u}{d}} \right)
\le 2 e^{-u}.
\]
Equivalently, for all $u \in (0,d)$,
\begin{equation}
\label{eq:fixed-x-tail}
\Pr\!\left( \left| x^\top A x \right|
          \ge C_0 \sqrt{\frac{u}{d}} \right)
\le 2 e^{-u}.
\end{equation}

Now restrict to thresholds $u = c_1 d v^2$ with $v \in (0,1)$ and $c_1>0$ small enough so that $u<d$. Then
\[
\sqrt{\frac{u}{d}} = \sqrt{c_1}\, v,
\]
and~\eqref{eq:fixed-x-tail} becomes
\[
\Pr\!\left( \left| x^\top A x \right|
          \ge C_1 v \right)
\le 2 \exp(-c_1 d v^2),
\]
for constants $C_1 = C_0 \sqrt{c_1}$ and $c_1>0$. Renaming $v$ to $u$ and $c_1$ to $c_2$, we obtain:

\begin{equation}
\label{eq:fixed-direction-subgaussian}
\Pr\!\left( \left| x^\top A x \right|
          \ge u \right)
\le 2 \exp(-c_2 d u^2),
\qquad \forall\, u \in (0,1),\ \forall\, x \in \mathbb{S}^{q-1},
\end{equation}
for some absolute constant $c_2>0$ and where $\mathbb{S}^{q-1} = \{x \in \mathbb{R}^q : \|x\|=1\}$.

Let $\mathbb{S}^{q-1}$ denote the unit sphere in $\mathbb{R}^q$. For $\varepsilon \in (0,1)$, an $\varepsilon$-net $\mathcal{N}_\varepsilon \subset \mathbb{S}^{q-1}$ is a finite set such that for any $x \in \mathbb{S}^{q-1}$ there exists $y \in \mathcal{N}_\varepsilon$ with $\|x-y\| \le \varepsilon$. It is well known that there exists an $\varepsilon$-net with cardinality bounded by
\[
|\mathcal{N}_\varepsilon|
\le \left(1 + \frac{2}{\varepsilon}\right)^q.
\]
 We fix $\varepsilon = \frac{1}{4}$ and let $\mathcal{N}$ be a $\frac{1}{4}$-net; then
\[
|\mathcal{N}| \le 9^q.
\]

We also use a standard net approximation for the operator norm of a symmetric matrix.

\begin{claim}
\label{clm:net-operator-norm}
Let $A$ be a symmetric $q\times q$ matrix and let $\mathcal{N}$ be a $\frac{1}{4}$-net of $\mathbb{S}^{q-1}$. Then
\[
\|A\|_{op}
= \sup_{x\in\mathbb{S}^{q-1}} |x^\top A x|
\le 2 \sup_{y\in\mathcal{N}} |y^\top A y|.
\]
\end{claim}

\begin{proof}[Proof of Claim~\ref{clm:net-operator-norm}]
Let $x^\star \in \mathbb{S}^{q-1}$ satisfy $|{x^\star}^\top A x^\star| = \|A\|_{op}$. Choose $y \in \mathcal{N}$ with $\|x^\star - y\| \le \tfrac{1}{4}$. Then
\[
x^\star = y + (x^\star - y),
\qquad \|x^\star - y\| \le \tfrac{1}{4}.
\]
Using symmetry of $A$,
\[
x^{\star\top} A x^\star
= y^\top A y + 2 (x^\star - y)^\top A y + (x^\star - y)^\top A (x^\star - y).
\]
Taking absolute values and using $\|y\|,\|x^\star-y\|\le 1$,
\begin{align*}
\|A\|_{op}
&\le |y^\top A y|
  + 2 \|x^\star - y\| \|A\|_{op} \|y\|
  + \|x^\star - y\|^2 \|A\|_{op} \\
&\le |y^\top A y| + \left(2\cdot \tfrac{1}{4}\cdot 1 + \tfrac{1}{4}^2\right)\|A\|_{op}
= |y^\top A y| + \tfrac{9}{16}\|A\|_{op}.
\end{align*}
Rearranging gives
\[
\frac{7}{16}\|A\|_{op} \le |y^\top A y|,
\]
hence
\[
\|A\|_{op} \le \frac{16}{7} \sup_{y\in\mathcal{N}} |y^\top A y| \le 2 \sup_{y\in\mathcal{N}} |y^\top A y|.
\]
\end{proof}

Applying Claim~\ref{clm:net-operator-norm} to $A = \Sigma_d - I_q$ yields
\[
\left\| \frac{1}{d} U^\top U - I_q \right\|_{op}
= \|A\|_{op}
\le 2 \max_{y\in\mathcal{N}} |y^\top A y|.
\]

Fix $\delta \in (0,1)$, to be chosen later. Using the net bound and~\eqref{eq:fixed-direction-subgaussian}, we have
\begin{align*}
\Pr\!\left( \|A\|_{op} \ge \delta \right)
&\le \Pr\!\left( \exists\, y \in \mathcal{N} \text{ such that } |y^\top A y| \ge \frac{\delta}{2} \right) \\
&\le \sum_{y\in\mathcal{N}} \Pr\!\left( |y^\top A y| \ge \frac{\delta}{2} \right) \\
&\le |\mathcal{N}| \cdot 2 \exp\!\left(-c_2 d \frac{\delta^2}{4}\right) \\
&\le 2 \exp\!\left( q \log 9 - c_3 d \delta^2 \right),
\end{align*}
for some absolute constant $c_3>0$.

Now take
\[
\delta
= C\Big(\sqrt{\frac{q}{d}} + t\Big),
\qquad t\in(0,1),
\]
for a sufficiently large absolute constant $C>0$ (to be fixed). Then
\[
\delta^2
= C^2\left(\frac{q}{d} + t^2 + 2 t \sqrt{\frac{q}{d}}\right)
\ge \frac{C^2}{2}\left(\frac{q}{d} + t^2\right),
\]
and hence
\begin{align*}
q \log 9 - c_3 d \delta^2
&\le q \log 9 - c_3 d \cdot \frac{C^2}{2}\left(\frac{q}{d} + t^2\right) \\
&= q\left(\log 9 - \frac{c_3 C^2}{2}\right)
   - \frac{c_3 C^2}{2} d t^2.
\end{align*}
Choose $C$ large enough so that $\log 9 - \frac{c_3 C^2}{2} \le -1$. Then
\[
q\left(\log 9 - \frac{c_3 C^2}{2}\right)
\le - q \le 0,
\]
and hence
\[
q \log 9 - c_3 d \delta^2
\le - c_4 d t^2
\]
for some absolute constant $c_4>0$. Plugging this into the previous bound gives
\[
\Pr\!\left(
  \left\| \frac{1}{d} U^\top U - I_q \right\|_{op}
  \ge C\Big(\sqrt{\frac{q}{d}} + t\Big)
\right)
\le 2 \exp(-c_4 d t^2),
\]
which proves the first statement (with $c := c_4$).

\textbf{The in particular statement.}
For the second claim, take $t = \sqrt{q/d}$ (note that $t \in (0,1)$ as soon as $d$ is sufficiently large compared to $q$). Then
\[
C\Big(\sqrt{\frac{q}{d}} + t\Big)
= 2 C \sqrt{\frac{q}{d}},
\]
and
\[
\Pr\!\left(
  \left\| \frac{1}{d} U^\top U - I_q \right\|_{op}
  \ge 2 C \sqrt{\frac{q}{d}}
\right)
\le 2 \exp\!\left(-c_4 d \frac{q}{d}\right)
= 2 \exp(-c_4 q).
\]
Thus, with probability at least $1 - 2\exp(-c_4 q)$ we have
\[
\left\| \frac{1}{d} U^\top U - I_q \right\|_{op}
\le 2 C \sqrt{\frac{q}{d}},
\]
which is the claimed $O(\sqrt{q/d})$ bound (with $c' := c_4$).
\end{proof}

Thus, in the regime $d \gg q$ the Gram matrix satisfies
\[
U^\top U
= d\,(I_q + \Delta_d),
\qquad
\|\Delta_d\|_{op}
= O\!\left(\sqrt{\frac{q}{d}}\right)
\]
with overwhelmingly high probability. Geometrically, the columns of $U$ are almost mutually orthogonal and have almost equal norms.

\subsection{Approximate Structure of the Projection Operator}
\label{app:high_dim_projection}

We now use Lemma~\ref{lem:gram_concentration} to approximate the projection operator $P_U$ in the high-dimensional regime. Recall that
\[
P_U
= U(U^\top U)^{-1} U^\top
= U \big(d(I_q+\Delta_d)\big)^{-1} U^\top
= \frac{1}{d} U (I_q+\Delta_d)^{-1} U^\top.
\]

\begin{lemma}[Approximation of $(U^\top U)^{-1}$]
\label{lem:inverse_expansion}
Suppose $\|\Delta_d\|_{op} \le \frac{1}{2}$. Then
\[
(I_q + \Delta_d)^{-1}
= I_q + E_d,
\qquad
\|E_d\|_{op}
\le 2 \|\Delta_d\|_{op},
\]
and hence
\[
(U^\top U)^{-1}
= \frac{1}{d} (I_q + E_d),
\qquad
\|(U^\top U)^{-1} - \tfrac{1}{d} I_q\|_{op}
\le \frac{2}{d} \|\Delta_d\|_{op}.
\]
\end{lemma}

\begin{proof}
If $\|\Delta_d\|_{op} \le \frac{1}{2}$ then the Neumann series
\[
(I_q+\Delta_d)^{-1} = \sum_{k=0}^{\infty} (-\Delta_d)^k
\]
converges in operator norm, and
\[
(I_q+\Delta_d)^{-1} - I_q
= \sum_{k=1}^{\infty} (-\Delta_d)^k.
\]
Using the triangle inequality and submultiplicativity yields
\[
\|E_d\|_{op}
= \left\| \sum_{k=1}^{\infty} (-\Delta_d)^k \right\|_{op}
\le \sum_{k=1}^{\infty} \|\Delta_d\|_{op}^k
= \frac{\|\Delta_d\|_{op}}{1-\|\Delta_d\|_{op}}
\le 2 \|\Delta_d\|_{op},
\]
where the last inequality uses $\|\Delta_d\|_{op} \le \frac{1}{2}$. The bound for $(U^\top U)^{-1}$ follows by multiplication by $1/d$.
\end{proof}

Combining Lemmas~\ref{lem:gram_concentration} and~\ref{lem:inverse_expansion}, we obtain that for $d$ sufficiently large and $q \ll d$, with high probability,
\[
(U^\top U)^{-1}
= \frac{1}{d} I_q + O\!\left(\frac{1}{d}\sqrt{\frac{q}{d}}\right)
\quad\text{in operator norm}.
\]

We next propagate this approximation through to the projector $P_U$.

\begin{proposition}[Approximation of $P_U$ by $UU^\top / d$]
\label{prop:proj_approx}
Let $U \in \mathbb{R}^{d\times q}$ have i.i.d.\ $\mathcal{N}(0,1)$ entries. There exist absolute constants $c,C>0$ such that the following holds. With probability at least $1 - 2\exp(-c q)$,
\[
\left\| P_U - \frac{1}{d} U U^\top \right\|_{op}
\;\le\; C \sqrt{\frac{q}{d}}.
\]
Consequently, for any $g \in \mathbb{R}^d$,
\[
\left\| P_U g - \frac{1}{d} U U^\top g \right\|
\;\le\; C \sqrt{\frac{q}{d}} \,\|g\|.
\]
\end{proposition}

\begin{proof}
Write
\[
P_U - \frac{1}{d} U U^\top
= U\Big((U^\top U)^{-1} - \tfrac{1}{d} I_q\Big) U^\top
= \frac{1}{d} U E_d U^\top,
\]
where $E_d$ is as in Lemma~\ref{lem:inverse_expansion}. Thus
\[
\left\| P_U - \frac{1}{d} U U^\top \right\|_{op}
\le \frac{1}{d}\|U\|_{op}^2 \,\|E_d\|_{op}.
\]

For a Gaussian matrix $U \in \mathbb{R}^{d\times q}$, standard random matrix bounds give
\[
\|U\|_{op}
\le \sqrt{d} + \sqrt{q} + t
\]
with probability at least $1 - 2\exp(-c' t^2)$ for some absolute constant $c'>0$. Taking $t$ of order $\sqrt{q}$ and using $q \le d$ yields $\|U\|_{op}^2 = O(d)$ with probability at least $1 - 2\exp(-c'' q)$ for some $c''>0$.

On the same high-probability event, Lemma~\ref{lem:gram_concentration} ensures that $\|\Delta_d\|_{op} = O(\sqrt{q/d})$ and in particular $\|\Delta_d\|_{op} \le \tfrac{1}{2}$ for $d$ sufficiently large. Lemma~\ref{lem:inverse_expansion} then gives
\[
\|E_d\|_{op}
\le 2 \|\Delta_d\|_{op}
= O\!\left(\sqrt{\frac{q}{d}}\right).
\]
Putting these bounds together,
\[
\left\| P_U - \frac{1}{d} U U^\top \right\|_{op}
\le \frac{1}{d} \cdot C_1 d \cdot C_2 \sqrt{\frac{q}{d}}
= C \sqrt{\frac{q}{d}}
\]
for suitable absolute constants $C_1,C_2,C>0$. The bound for $P_U g - (1/d) U U^\top g$ follows immediately by multiplying by $\|g\|$.
\end{proof}

Proposition~\ref{prop:proj_approx} formalizes the statement that, when $d \gg q$, the action of $P_U$ on any vector $g$ is very close (in norm) to the action of the simpler operator $(1/d) U U^\top$. In other words, within the sampled subspace, the update direction produced by ZO-Align is, up to a small relative error of order $\sqrt{q/d}$, the same as ZO-Avg.

\subsection{Implications for ZO-Align in High Dimensions}
\label{app:high_dim_implications}

The main advantage of ZO-Align over ZO-AVG is that it explicitly accounts for correlations among the sampled directions through the factor $(U^\top U)^{-1}$, and thus adjusts the relative contributions of different directions in $U$ according to the spectrum of $U^\top U$. When $d$ is moderate and $q$ is a nontrivial fraction of $d$, the Gram matrix $U^\top U$ can deviate substantially from a scaled identity, and these spectral corrections can significantly improve how the sampled subspace represents the local geometry of $f$.

In the high-dimensional regime $d \gg q$, however, Lemma~\ref{lem:gram_concentration} and Proposition~\ref{prop:proj_approx} show that, with high probability,
\[
U^\top U \approx d I_q,
\qquad
P_U \approx \frac{1}{d} U U^\top,
\]
with approximation error of order $\sqrt{q/d}$ in operator norm. In this case the sampled directions are already almost mutually orthogonal with nearly equal norms, and the spectral corrections encoded in $(U^\top U)^{-1}$ reduce, to leading order, to a global scaling by $1/d$. As a consequence, the additional structural advantage of ZO-Align over naive aggregation inside the sampled subspace becomes much less pronounced when $d$ is extremely large and $q \ll d$.

It is important to emphasize that this high-dimensional geometric simplification does not alter the exact MSE formulas in Proposition~\ref{prop:mse}, which are valid for every finite $d$. The high-dimensional analysis merely explains why the \emph{geometric} distinction inside the sampled subspace becomes weaker as $d/q \to \infty$.

\subsection{Numerical Illustration of \texorpdfstring{$U^\top U \approx d I_q$}{UTU ≈ dIq}}
\label{app:high_dim_figure}

To empirically illustrate the concentration of the Gram matrix, we consider Gaussian query matrices $U \in \mathbb{R}^{d\times q}$ with a fixed $q=10$ and varying dimension $d \in \{10, 100, 1000, 10000\}$. For each value of $d$, we sample $U$ with i.i.d.\ $\mathcal{N}(0,1)$ entries and compute the normalized Gram matrix
\[
G_d \;=\; \frac{1}{d} U^\top U \in \mathbb{R}^{q\times q}.
\]
We then visualize the absolute values $|G_d|$ as heatmaps. When $d$ is small, $G_d$ exhibits noticeable deviations from the identity, with substantial off-diagonal entries. As $d$ increases, the diagonal entries of $G_d$ stabilize near $1$, while the off-diagonal entries shrink toward $0$, and the heatmaps become visually indistinguishable from those of an identity matrix.

A typical realization is shown in Figure~\ref{fig:gram_heatmaps}. The progressive disappearance of off-diagonal structure as $d$ grows is consistent with Lemma~\ref{lem:gram_concentration} and supports the approximation $U^\top U \approx d I_q$ used in the high-dimensional discussion above.

\begin{figure}[t]
    \centering
    \includegraphics[width=\textwidth]{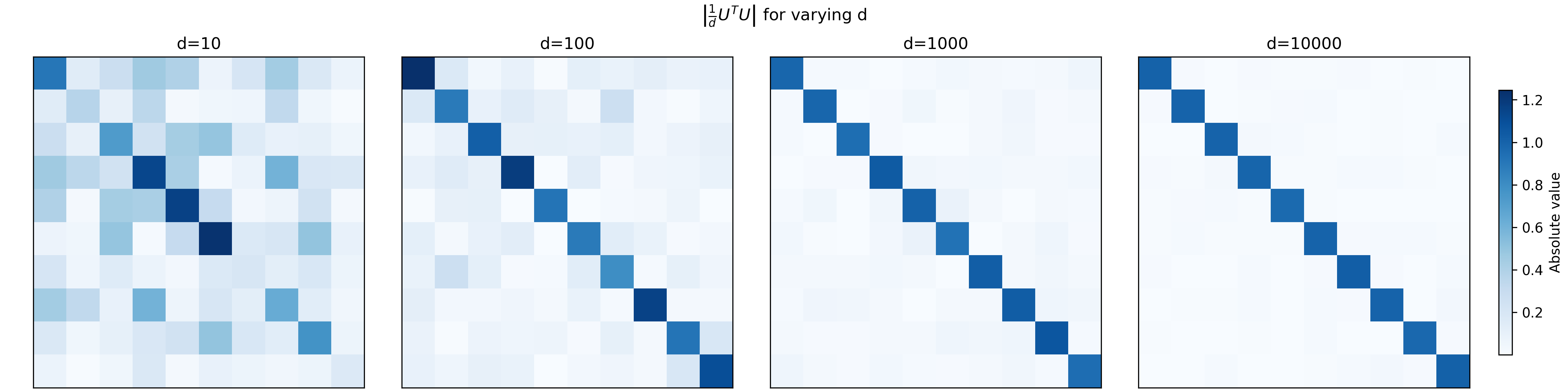}
    \caption{Heatmaps of the absolute value of the normalized Gram matrix $G_d = (1/d) U^\top U$ for $q=10$ and varying dimensions $d \in \{10, 100, 1000, 10000\}$. Each $U$ has i.i.d.\ $\mathcal{N}(0,1)$ entries. As $d$ increases, $G_d$ becomes visually closer to the identity matrix, with diagonal entries near $1$ and off-diagonal entries shrinking toward $0$.}
    \label{fig:gram_heatmaps}
\end{figure}

\section{Proofs}





\subsection{Proof of Proposition~\ref{prop:mse}}
\begin{proof}
\label{Proof:prop_mse}

\noindent\textbf{Analysis of ZO-Avg.}
Since ZO-Avg is unbiased, its MSE is the trace of its covariance matrix. Let $g = \nabla f(x)$ and $\hat{g}_i = (g^T u_i)u_i$. The ZO-Avg estimator is $\hat{g}_{AVG} = \frac{1}{q} \sum_{i=1}^q \hat{g}_i$. As the directions $u_i \sim \mathcal{N}(0, I)$ are i.i.d., the terms $\hat{g}_i$ are also i.i.d. random vectors. The variance of the average is thus:
\begin{equation*}
    \text{Var}(\hat{g}_{AVG}) = \frac{1}{q^2} \sum_{i=1}^q \text{Var}(\hat{g}_i) = \frac{1}{q} \text{Var}(\hat{g}_1).
\end{equation*}
To compute $\text{Var}(\hat{g}_1)$, we need the first two moments of $\hat{g}_1$. We have $\mathbb{E}[\hat{g}_1] = \mathbb{E}[u_1 u_1^T]g = I g = g$. For the second moment, we use the known result for a standard Gaussian vector $u \sim \mathcal{N}(0, I)$ and a fixed vector $v$: $\mathbb{E}[(v^T u)^2 u u^T] = 2vv^T + \|v\|^2 I$. Thus:
\begin{equation*}
    \mathbb{E}[\hat{g}_1 \hat{g}_1^T] = \mathbb{E}[(g^T u_1)^2 u_1 u_1^T] = 2gg^T + \|g\|^2 I.
\end{equation*}
The covariance matrix of a single estimator is:
\begin{align*}
    \text{Var}(\hat{g}_1) &= \mathbb{E}[\hat{g}_1 \hat{g}_1^T] - \mathbb{E}[\hat{g}_1]\mathbb{E}[\hat{g}_1]^T \\
    &= (2gg^T + \|g\|^2 I) - gg^T = gg^T + \|g\|^2 I.
\end{align*}
The covariance of the ZO-Avg estimator is $\text{Var}(\hat{g}_{AVG}) = \frac{1}{q}(gg^T + \|g\|^2 I)$. The MSE is the trace of this matrix:
\begin{align*}
    \text{MSE}(\hat{g}_{AVG}) &= \text{tr}(\text{Var}(\hat{g}_{AVG})) = \frac{1}{q} \text{tr}(gg^T + \|g\|^2 I) \\
    &= \frac{1}{q} (\text{tr}(gg^T) + \|g\|^2 \text{tr}(I)) = \frac{1}{q}(\|g\|^2 + d\|g\|^2) = \frac{d+1}{q} \|g\|^2.
\end{align*}
\noindent\textbf{Analysis of ZO-Align.}
The ZO-Align estimator, $\hat{g}_{ALGN} = P_U g$, is biased, with $\mathbb{E}[\hat{g}_{ALG}] = \frac{q}{d}g$. Here, we compute the MSE directly. Let $\hat{g}_{U} = P_U g$.
\begin{align*}
    \text{MSE}(\hat{g}_{ALG}) &= \mathbb{E}[\|\hat{g}_{U} - g\|^2] = \mathbb{E}[\|(P_U - I)g\|^2] \\
    &= \mathbb{E}[g^T (P_U - I)^T (P_U - I) g].
\end{align*}
Since $P_U$ is a projection matrix, it is symmetric ($P_U^T = P_U$) and idempotent ($P_U^2 = P_U$). Therefore, $(P_U - I)^2 = P_U^2 - 2P_U + I = P_U - 2P_U + I = I - P_U$.
\begin{align*}
    \text{MSE}(\hat{g}_{ALG}) &= \mathbb{E}[g^T(I - P_U)g] = g^T \mathbb{E}[I - P_U] g \\
    &= g^T (I - \mathbb{E}[P_U]) g.
\end{align*}
Using the known result that the expectation of a projection onto a random $q$-dimensional subspace is $\mathbb{E}[P_U] = \frac{q}{d}I$, we have:
\begin{equation*}
    \text{MSE}(\hat{g}_{ALIGN}) = g^T (I - \frac{q}{d}I) g = \left(1 - \frac{q}{d}\right) \|g\|^2 = \frac{d-q}{d} \|g\|^2.
\label{eq:mse_align}
\end{equation*}
\end{proof}

\subsection{Proof of Proposition~\ref{prop:optimal_stepsizes}}\label{Proof:prop_optimal_stepsizes}
\begin{proof}
    From the $L$-smoothness of $f$, we have the descent lemma:
$$f(x_{t+1}) \le f(x_t) + \langle \nabla f(x_t), x_{t+1} - x_t \rangle + \frac{L}{2} \|x_{t+1} - x_t\|^2.$$ 
For ZO-AVG, Substituting the update rule $x_{t+1} - x_t = -\eta_t \hat{g}_{AVG}(x_t)$ and taking the expectation conditional on $x_t$:
$$\mathbb{E}[f(x_{t+1}) | x_t] \le f(x_t) - \eta_t \langle \nabla f(x_t), \mathbb{E}[\hat{g}_{AVG}(x_t)] \rangle + \frac{L\eta_t^2}{2} \mathbb{E}[\|\hat{g}_{AVG}(x_t)\|^2]$$
Using the properties $\mathbb{E}[\hat{g}_{AVG}(x_t)] = \nabla f(x_t)$ and $\mathbb{E}[\|\hat{g}_{AVG}(x_t)\|^2] = \frac{q_t+d+1}{q_t}\|\nabla f(x_t)\|^2$, we get:
$$
\mathbb{E}[f(x_{t+1}) | x_t] \le f(x_t) - \eta_t \|\nabla f(x_t)\|^2 + \frac{L\eta_t^2(q_t+d+1)}{2q_t} \|\nabla f(x_t)\|^2.  
$$
It is easy to see that the optimal step size $\eta_t^* = \frac{q_t}{L(q_t+d+1)}$.

Similarly, for ZO-Align, using the properties $\mathbb{E}[\hat{g}_{ALG}(x_t)]=\frac{q_t}{d}\nabla f(x_t)$ and $\mathbb{E}[\|\hat{g}_{ALG}(x_t)\|^2]=\frac{q_t}{d}\|\nabla f(x_t)\|^2$, we have:
$$
\mathbb{E}[f(x_{t+1})|x_t] \le f(x_t) - \frac{q_t}{d}\eta_t\|\nabla f(x_t)\|^2+ \frac{L\eta^2_tq_t}{2d}\|\nabla f(x_t)\|^2. 
$$
It is easy to see that the optimal step size $\eta_t^*=1/L$.
\end{proof}

\subsection{Proof of Theorem~\ref{thm:zo_avg_strong_convex}}
\begin{proof}
\label{Proof:thm_zo_avg_strongly_convex}
The first part of the proof establishes the one-step contraction factor. From the $L$-smoothness of $f$, we have the descent lemma:
$$f(x_{t+1}) \le f(x_t) + \langle \nabla f(x_t), x_{t+1} - x_t \rangle + \frac{L}{2} \|x_{t+1} - x_t\|^2$$
Substituting the update rule $x_{t+1} - x_t = -\eta_t \hat{g}_{AVG}(x_t)$ and taking the expectation conditional on $x_t$:
$$\mathbb{E}[f(x_{t+1}) | x_t] \le f(x_t) - \eta_t \langle \nabla f(x_t), \mathbb{E}[\hat{g}_{AVG}(x_t)] \rangle + \frac{L\eta_t^2}{2} \mathbb{E}[\|\hat{g}_{AVG}(x_t)\|^2]$$
Using the properties $\mathbb{E}[\hat{g}_{AVG}(x_t)] = \nabla f(x_t)$ and $\mathbb{E}[\|\hat{g}_{AVG}(x_t)\|^2] = \frac{q_t+d+1}{q_t}\|\nabla f(x_t)\|^2$, we get:
$$
\mathbb{E}[f(x_{t+1}) | x_t] \le f(x_t) - \eta_t \|\nabla f(x_t)\|^2 + \frac{L\eta_t^2(q_t+d+1)}{2q_t} \|\nabla f(x_t)\|^2.  
$$
It is easy to see that the optimal step size $\eta_t^* = \frac{q_t}{L(q_t+d+1)}$. Plugging in the optimal step size, we obtain the one-step progress:
$$\mathbb{E}[f(x_{t+1}) | x_t] \le f(x_t) - \frac{q_t}{2L(q_t+d+1)}\|\nabla f(x_t)\|^2$$
By $\gamma$-strong convexity, $\|\nabla f(x_t)\|^2 \ge 2\gamma(f(x_t) - f(x_*))$. Let $\delta_t = E[f(x_t) - f(x_*)]$. Taking the full expectation and applying the convexity condition gives:
$$\delta_{t+1} \le \left(1 - \frac{\gamma q_t}{L(q_t+d+1)}\right)\delta_t$$
Unrolling this recurrence from $t=T-1$ down to $0$:
$$\delta_T \le \left(1 - \frac{\gamma q_{T-1}}{L(q_{T-1}+d+1)}\right) \delta_{T-1}$$
$$\le \left(1 - \frac{\gamma q_{T-1}}{L(q_{T-1}+d+1)}\right) \left(1 - \frac{\gamma q_{T-2}}{L(q_{T-2}+d+1)}\right) \delta_{T-2}$$
$$\dots \le \left( \prod_{t=0}^{T-1} \left(1 - \frac{\gamma q_t}{L(q_t+d+1)}\right) \right) \delta_0$$
This establishes the first part of the theorem.

For the second part, we must find the sequence $\{q_t\}_{t=0}^{T-1}$ that minimizes the final error bound under the budget constraint. This is equivalent to minimizing the total contraction factor. Let $\rho_t = \frac{\gamma q_t}{L(q_t+d+1)}$. We solve the following optimization problem:
$$\min_{\{q_t\}, T} \prod_{t=0}^{T-1} (1 - \rho_t) \quad \text{s.t.} \sum_{t=0}^{T-1} q_t \le K, \quad 1 \le q_t \le d, \quad q_t \in \mathbb{Z}$$
Minimizing the product is equivalent to minimizing its logarithm, $\sum_{t=0}^{T-1} \log(1 - \rho_t)$. Since $\log(x)$ is a strictly increasing function, this is equivalent to maximizing $\sum_{t=0}^{T-1} \rho_t$. Let $h(q) = \frac{q}{q+d+1}$. The problem becomes:
$$\max_{\{q_t\}, T} \sum_{t=0}^{T-1} \frac{\gamma}{L} h(q_t) \quad \text{s.t.} \sum_{t=0}^{T-1} q_t \le K, \quad 1 \le q_t \le d, \quad q_t \in \mathbb{Z}$$
The function $h(q)$ is concave for $q>0$, as its second derivative $h''(q) = \frac{-2(d+1)}{(q+d+1)^3}$ is negative. By Jensen's inequality, for a fixed number of iterations $T$, the sum $\sum_{t=0}^{T-1} h(q_t)$ is maximized when the $q_t$ values are as equal as possible. Furthermore, because $h(q)/q$ is a decreasing function of $q$, we gain more "value per query" for smaller $q$. This implies we should maximize the number of terms in the sum, i.e., maximize $T$. The maximum possible $T$ is $K$, which occurs when we choose $q_t=1$ for all $t=0,\dots,K-1$. This strategy satisfies both the concavity argument (all $q_t$ are equal) and the budget constraint, and is therefore optimal.
\end{proof}

\subsection{Proof of Theorem~\ref{thm:zo_align_strong_convex}}
\begin{proof}
\label{Proof:thm_zo_align_strongly_convex}
As in~\ref{Proof:thm_zo_avg_strongly_convex}, we establish the one-step contraction factor. Using the properties $\mathbb{E}[\hat{g}_{ALG}(x_t)]=\frac{q_t}{d}\nabla f(x_t)$ and $\mathbb{E}[\|\hat{g}_{ALG}(x_t)\|^2]=\frac{q_t}{d}\|\nabla f(x_t)\|^2$, we have:
$$
\mathbb{E}[f(x_{t+1})|x_t] \le f(x_t) - \frac{q_t}{d}\eta_t\|\nabla f(x_t)\|^2+ \frac{L\eta^2_tq_t}{2d}\|\nabla f(x_t)\|^2. 
$$
It is easy to see that the optimal step size $\eta_t^*=1/L$. Plugging in the optimal step size, we have:
$$\mathbb{E}[f(x_{t+1})|x_t] \le f(x_t) - \frac{q_t}{2Ld}\|\nabla f(x_t)\|^2$$
Applying the $\gamma$-strong convexity condition and taking the full expectation yields:
$$E[f(x_{t+1}) - f(x_*)] \le \left(1 - \frac{\gamma q_t}{Ld}\right)(E[f(x_t)] - f(x_*))$$
Unrolling this recurrence, just as in the proof of Theorem~\ref{thm:zo_avg_strong_convex}, gives the first result.

For the second part, we must find the allocation $\{q_t\}$ that minimizes the final error bound. This is equivalent to solving the optimization problem:
$$\min_{\{q_t\},T} \prod_{t=0}^{T-1} \left(1 - \frac{\gamma q_t}{Ld}\right) \quad \text{s.t.} \sum_{t=0}^{T-1} q_t \le K, \quad 1 \le q_t \le d, \quad q_t \in \mathbb{Z}$$
Let $c=\gamma/(Ld)$. The objective function $F(q_0,\dots,q_{T-1})=\prod_{t=0}^{T-1}(1-cq_t)$ is a product of linear terms. Without loss of geneity, it is equivalent to minimize its logarithm, $ln(F(q))$. The logarithm function is concave over the feasible region (the simplex defined by the budget constraint). A fundamental result in convex optimization is that the minimum of a concave function over a compact convex set is achieved at an extreme point (a vertex) of the set.

The vertices of the allocation simplex correspond to solutions where the budget is concentrated as much as possible. Given the constraints, the extreme points are strategies that use the largest possible query count, $q_t=d$, for some number of iterations, and $q_t=0$ for others. To satisfy the budget $\sum q_t=K$, the optimal strategy is therefore to perform $T=K/d$ iterations (assuming $K$ is a multiple of $d$) with $q_t=d$ for each iteration. This concentrates the entire query budget into the fewest, most powerful steps, thus minimizing the product of contraction factors.
\end{proof}

\subsection{Proof of Theorem~\ref{thm:zo_avg_convex}}
\begin{proof}
\label{Proof:thm_zo_avg_convex}
We analyze the evolution of the squared distance to the optimum, $\|x_t - x_*\|^2$.
$$\|x_{t+1} - x_*\|^2 = \|x_t - x_*\|^2 - 2\eta_t \hat{g}_{AVG}(x_t)^T (x_t - x_*) + \eta_t^2 \|\hat{g}_{AVG}(x_t)\|^2$$
Taking the expectation conditional on $x_t$ and using the properties of the ZO-Avg estimator, we get:
\begin{equation*}
    E[\|x_{t+1} - x_*\|^2 | x_t] = \|x_t - x_*\|^2 - 2\eta_t \nabla f(x_t)^T (x_t - x_*) + \eta_t^2 \frac{q_t+d+1}{q_t} \|\nabla f(x_t)\|^2
\end{equation*}
By convexity, we have $\nabla f(x_t)^T(x_t - x_*) \ge f(x_t) - f(x_*)$. Substituting the step size $\eta_t = \frac{q_t}{L(q_t+d+1)}$ yields:
\begin{equation*}
    E[\|x_{t+1} - x_*\|^2 | x_t] \le \|x_t - x_*\|^2 - \frac{2q_t}{L(q_t+d+1)}(f(x_t) - f(x_*)) + \frac{q_t}{L^2(q_t+d+1)} \|\nabla f(x_t)\|^2
\end{equation*}
From the one-step analysis in Theorem~\ref{thm:zo_avg_strong_convex}, we know that $E[f(x_t) - f(x_{t+1}) | x_t] \ge \frac{q_t}{2L(q_t+d+1)}\|\nabla f(x_t)\|^2$. This implies $\frac{q_t}{L(q_t+d+1)}\|\nabla f(x_t)\|^2 \le 2(f(x_t) - E[f(x_{t+1})|x_t])$. Substituting this into the inequality gives:
\begin{equation*}
    E[\|x_{t+1} - x_*\|^2 | x_t] \le \|x_t - x_*\|^2 - \frac{2q_t}{L(q_t+d+1)}(f(x_t) - f(x_*)) + \frac{2}{L}(f(x_t) - E[f(x_{t+1})|x_t])
\end{equation*}
Rearranging and taking the full expectation, we let $\delta_t = E[f(x_t) - f(x_*)]$ and $D_t = E[\|x_t - x_*\|^2]$:
\begin{equation*}
    \frac{2q_t}{L(q_t+d+1)}\delta_t \le D_t - D_{t+1} + \frac{2}{L}(\delta_t - \delta_{t+1})
\end{equation*}
Summing from $t=0$ to $T-1$:
\begin{equation*}
    \sum_{t=0}^{T-1} \frac{2q_t}{L(q_t+d+1)}\delta_t \le D_0 - D_T + \frac{2}{L}(\delta_0 - \delta_T) \le \|x_0 - x_*\|^2 + \frac{2}{L}(f(x_0) - f(x_*))
\end{equation*}
Since $\delta_t$ is a non-increasing sequence, $\delta_t \ge \delta_T$. Therefore:
\begin{equation*}
    \delta_T \left(\sum_{t=0}^{T-1} \frac{2q_t}{L(q_t+d+1)}\right) \le \|x_0 - x_*\|^2 + \frac{2}{L}(f(x_0) - f(x_*))
\end{equation*}
Rearranging gives the first result. To obtain the tightest bound on $\delta_T$, we must maximize the sum in the denominator, $\sum_{t=0}^{T-1} \frac{q_t}{q_t+d+1}$, subject to $\sum q_t \le K$. As established in the proof of Theorem~\ref{thm:zo_avg_strong_convex}, the function $h(q)=q/(q+d+1)$ is concave, and the value per query, $h(q)/q$, is maximized at $q=1$. Thus, the optimal strategy is to set $T=K$ and $q_t=1$ for all $t$. This yields the second result:
\begin{align*}
     E[f(x_K) - f(x_*)] &\le \frac{L(\|x_0 - x_*\|^2 + \frac{2}{L}(f(x_0) - f(x_*)))}{\sum_{t=0}^{K-1} \frac{2}{d+2}} \\
    &= \frac{L(d+2)(\|x_0 - x_*\|^2 + \frac{2}{L}(f(x_0) - f(x_*)))}{2K}
\end{align*}
This completes the proof.
\end{proof}

\subsection{Proof of Theorem~\ref{thm:zo_align_convex}}
\begin{proof}
\label{Proof:thm_zo_align_convex}
Following the same initial steps as in the proof of Theorem~\ref{thm:zo_avg_convex}, we start with the squared distance:
\begin{equation*}
    E[\|x_{t+1} - x_*\|^2 | x_t] = \|x_t - x_*\|^2 - 2\eta_t E[\hat{g}_{ALG}(x_t)]^T (x_t - x_*) + \eta_t^2 E[\|\hat{g}_{ALG}(x_t)\|^2]
\end{equation*}
Using the properties of ZO-Align ($E[\hat{g}_{ALG}]=\frac{q_t}{d}\nabla f(x_t)$, $E[\|\hat{g}_{ALG}\|^2]=\frac{q_t}{d}\|\nabla f(x_t)\|^2$), step size $\eta_t=1/L$, and the convexity property $\nabla f(x_t)^T(x_t - x_*) \ge f(x_t) - f(x_*)$:
\begin{equation*}
    E[\|x_{t+1} - x_*\|^2 | x_t] \le \|x_t - x_*\|^2 - \frac{2q_t}{Ld}(f(x_t) - f(x_*)) + \frac{q_t}{L^2d} \|\nabla f(x_t)\|^2
\end{equation*}
From the one-step analysis in Theorem~\ref{thm:zo_align_strong_convex}, we have $E[f(x_t) - f(x_{t+1}) | x_t] \ge \frac{q_t}{2Ld}\|\nabla f(x_t)\|^2$, which implies $\frac{q_t}{Ld}\|\nabla f(x_t)\|^2 \le 2(f(x_t) - E[f(x_{t+1})|x_t])$. Substituting this gives:
\begin{equation*}
    E[\|x_{t+1} - x_*\|^2 | x_t] \le \|x_t - x_*\|^2 - \frac{2q_t}{Ld}(f(x_t) - f(x_*)) + \frac{2}{L}(f(x_t) - E[f(x_{t+1})|x_t])
\end{equation*}
Rearranging, taking the full expectation, and summing from $t=0$ to $T-1$:
\begin{equation*}
    \sum_{t=0}^{T-1} \frac{2q_t}{Ld} E[f(x_t) - f(x_*)] \le \|x_0 - x_*\|^2 + \frac{2}{L}(f(x_0) - f(x_*))
\end{equation*}
By convexity of $f$, the sequence $E[f(x_t) - f(x_*)]$ is non-increasing. Thus, we can lower bound the left side using the final error, $E[f(x_T) - f(x_*)]$:
\begin{equation*}
    E[f(x_T) - f(x_*)] \left(\sum_{t=0}^{T-1} \frac{2q_t}{Ld}\right) \le \|x_0 - x_*\|^2 + \frac{2}{L}(f(x_0) - f(x_*))
\end{equation*}
Rearranging gives the final bound. 
\end{proof}

\subsection{Proof of Theorem~\ref{thm:zo_avg_nonconvex}}
\begin{proof}
\label{Proof:thm_zo_avg_nonconvex}
From the $L$-smoothness of $f$, we have the one-step progress from the proof of Theorem~\ref{thm:zo_avg_strong_convex}:
\begin{equation*}
E[f(x_{t+1})|x_t] \le f(x_t) - \frac{q_t}{2L(q_t+d+1)}\|\nabla f(x_t)\|^2
\end{equation*}
Taking the full expectation and rearranging gives:
\begin{equation*}
\frac{q_t}{2L(q_t+d+1)}\mathbb{E}[\|\nabla f(x_t)\|^2] \le \mathbb{E}[f(x_t)] - \mathbb{E}[f(x_{t+1})]
\end{equation*}
Summing from $t=0$ to $T-1$:
\begin{equation*}
\sum_{t=0}^{T-1} \frac{q_t}{2L(q_t+d+1)} \mathbb{E}[\|\nabla f(x_t)\|^2] \le \sum_{t=0}^{T-1} (\mathbb{E}[f(x_t)] - \mathbb{E}[f(x_{t+1})])
\end{equation*}
The right-hand side is a telescoping sum, which simplifies to $\mathbb{E}[f(x_0)] - \mathbb{E}[f(x_T)] \le f(x_0)-f^*$. For the left-hand side, we have:
\begin{equation*}
\left(\min_{t=0,...,T-1}\mathbb{E}[\|\nabla f(x_t)\|^2]\right) \left(\sum_{t=0}^{T-1} \frac{q_t}{2L(q_t+d+1)}\right) \le \sum_{t=0}^{T-1} \frac{q_t}{2L(q_t+d+1)} \mathbb{E}[\|\nabla f(x_t)\|^2]
\end{equation*}
Combining these gives the first result. To optimize the bound under a fixed budget $K$, we must maximize the denominator $\sum_{t=0}^{T-1}\frac{q_t}{q_t+d+1}$. As established in Theorem~\ref{thm:zo_avg_strong_convex}, this sum is maximized by choosing $q_t=1$ for $T=K$ iterations. Substituting this into the bound yields the final rate.
\end{proof}

\subsection{Proof of Theorem~\ref{thm:zo_align_nonconvex}}
\begin{proof}
\label{Proof:thm_zo_align_nonconvex}
The proof follows the same structure as for Theorem~\ref{thm:zo_avg_nonconvex}. From the one-step analysis in Theorem~\ref{thm:zo_align_strong_convex}, we have:
\begin{equation*}
E[f(x_{t+1})|x_t] \le f(x_t) - \frac{q_t}{2Ld}\|\nabla f(x_t)\|^2
\end{equation*}
Taking the full expectation, rearranging, and summing from $t=0$ to $T-1$:
\begin{equation*}
\sum_{t=0}^{T-1} q_t E[\|\nabla f(x_t)\|^2] \le 2Ld \sum_{t=0}^{T-1} (E[f(x_t)] - E[f(x_{t+1})]) \le 2Ld(f(x_0)-f^*)
\end{equation*}
Dividing by $\sum q_t$ and using the fact that $\min_t E[\|\nabla f(x_t)\|^2] \le \frac{1}{T}\sum E[\|\nabla f(x_t)\|^2]$ gives the result. The rate depends directly on the total number of queries, $\sum q_t$, which is maximized by using the full budget $K$.
\end{proof}

\subsection{Proof of Lemma~\ref{lem:sto_lsmooth}}
\begin{proof}
\label{Proof:lem_sto_lsmooth}
By the triangle inequality and the property $\|a+b\|^2 \le 2\|a\|^2 + 2\|b\|^2$, we have:
\begin{equation*}
\E_\xi[\|\nabla F(x, \xi)\|^2] \le 2\E_\xi[\|\nabla F(x, \xi) - \nabla F(x^*, \xi)\|^2] + 2\E_\xi[\|\nabla F(x^*, \xi)\|^2].
\end{equation*}
The second term is bounded by $2\sigma^2$ (or $2\E[\|\nabla F(x^*, \xi)\|^2]$) based on assumption~\ref{asm:bounded_variance}.
For the first term, we utilize the property of convex and $L$-smooth functions. Specifically, for any function $g$ that is convex and $L$-smooth, we have $\|\nabla g(x) - \nabla g(y)\|^2 \le 2L(g(x) - g(y) - \langle \nabla g(y), x-y \rangle)$. Applying this to $F(\cdot, \xi)$ with $y=x^*$:
\begin{equation*}
\|\nabla F(x, \xi) - \nabla F(x^*, \xi)\|^2 \le 2L \big( F(x, \xi) - F(x^*, \xi) - \langle \nabla F(x^*, \xi), x - x^* \rangle \big).
\end{equation*}
Taking the expectation with respect to $\xi$:
\begin{equation*}
\begin{aligned}
\E_\xi[\|\nabla F(x, \xi) - \nabla F(x^*, \xi)\|^2] &\le 2L \E_\xi \big[ F(x, \xi) - F(x^*, \xi) - \langle \nabla F(x^*, \xi), x - x^* \rangle \big] \\
&= 2L \big( f(x) - f(x^*) - \langle \nabla f(x^*), x - x^* \rangle \big).
\end{aligned}
\end{equation*}
Since $x^*$ is the minimizer of $f(x)$, we have $\nabla f(x^*) = 0$, causing the inner product term to vanish. Thus:
\begin{equation*}
\E_\xi[\|\nabla F(x, \xi) - \nabla F(x^*, \xi)\|^2] \le 2L(f(x)-f(x^*)).
\end{equation*}
Substituting this back into the first inequality yields the result:
\begin{equation*}
\E_\xi[\|\nabla F(x, \xi)\|^2] \le 2(2L(f(x)-f(x^*))) + 2\sigma^2 = 4L(f(x)-f(x^*)) + 2\sigma^2.
\end{equation*}
\end{proof}

\subsection{Proof of Theorem~\ref{thm:zo_avg_stochastic}}
\begin{proof}
\label{Proof:thm_zo_avg_sto}
We start by analyzing the evolution of the squared distance to the optimum, $\|x_t - x^*\|^2$. The update rule is $x_{t+1} = x_t - \eta_t \hat{g}_{AVG}(x_t)$.
\begin{equation*}
\|x_{t+1} - x^*\|^2 = \|x_t - x^*\|^2 - 2\eta_t \hat{g}_{AVG}(x_t)^T(x_t - x^*) + \eta_t^2 \|\hat{g}_{AVG}(x_t)\|^2
\end{equation*}
Let $\E_t[\cdot] = \E[\cdot | x_t]$ denote the expectation over the random direction matrix $U_t$ and the stochastic sample $\xi_t$. We have $\E_t[\hat{g}_{AVG}(x_t)] = \nabla f(x_t)$. From Proposition~\ref{prop:mse} and Lemma~\ref{lem:sto_lsmooth}, the expected squared norm of the estimator is bounded as:
\begin{equation*}
\E_t[\|\hat{g}_{AVG}(x_t)\|^2] = \frac{q_t+d+1}{q_t} \E_{\xi_t}[\|\nabla F(x_t, \xi_t)\|^2] \le \frac{q_t+d+1}{q_t}(4L(f(x_t) - f(x^*)) + 2\sigma^2)
\end{equation*}
Taking the conditional expectation of the distance expansion and using the convexity of $f$, i.e., $f(x_t) - f(x^*) \le \nabla f(x_t)^T(x_t - x^*)$:
\begin{align*}
\E_t[\|x_{t+1} - x^*\|^2] &\le \|x_t - x^*\|^2 - 2\eta_t \nabla f(x_t)^T(x_t - x^*) + \eta_t^2 \E_t[\|\hat{g}_{AVG}(x_t)\|^2] \\
&\le \|x_t - x^*\|^2 - 2\eta_t (f(x_t) - f(x^*)) + \eta_t^2 \frac{q_t+d+1}{q_t}(4L(f(x_t) - f(x^*)) + 2\sigma^2) \\
&= \|x_t - x^*\|^2 + \left(4L\eta_t^2\frac{q_t+d+1}{q_t} - 2\eta_t\right)(f(x_t) - f(x^*)) + 2\eta_t^2\frac{q_t+d+1}{q_t}\sigma^2
\end{align*}

To ensure convergence, we require the coefficient of the suboptimality gap to be effectively negative. Specifically, we enforce:
$$4L\eta_t^2 \frac{q_t+d+1}{q_t} - 2\eta_t \le -\eta_t$$$$4L\eta_t^2 \frac{q_t+d+1}{q_t} \le \eta_t$$$$\eta_t \le \frac{q_t}{4L(q_t+d+1)}$$

Since $\frac{q_t}{q_t+d+1}$ is increasing in $q_t$, the tightest constraint occurs at $q_t=1$, requiring $\eta_t \le \frac{1}{4L(d+2)}$. Our choice of $\eta_0$ ensures this holds for all $t$. Substituting this back into the inequality:$$\mathbb{E}_t[||x_{t+1} - x^*||^2] \le ||x_t - x^*||^2 - \eta_t(f(x_t) - f(x^*)) + 2\eta_t^2 \frac{q_t+d+1}{q_t}\sigma^2$$Taking the full expectation, rearranging to isolate the gap $\delta_t = \mathbb{E}[f(x_t) - f(x^*)]$, and summing from $t=0$ to $T-1$:$$\sum_{t=0}^{T-1} \eta_t \delta_t \le \mathbb{E}[||x_0 - x^*||^2] - \mathbb{E}[||x_T - x^*||^2] + 2\sigma^2 \sum_{t=0}^{T-1} \eta_t^2 \frac{q_t+d+1}{q_t}$$$$\le ||x_0 - x^*||^2 + 2\sigma^2 \sum_{t=0}^{T-1} \eta_t^2 \frac{q_t+d+1}{q_t}$$

By convexity and Jensen's inequality, for the weighted average iterate $\overline{x}_T$ with weights proportional to $\eta_t$:$$\mathbb{E}[f(\overline{x}_T) - f(x^*)] \le \frac{\sum_{t=0}^{T-1} \eta_t \delta_t}{\sum_{t=0}^{T-1} \eta_t} \le \frac{||x_0 - x^*||^2 + 2\sigma^2 \sum_{t=0}^{T-1} \eta_t^2 \frac{q_t+d+1}{q_t}}{\sum_{t=0}^{T-1} \eta_t}$$Substituting $\eta_t = \eta_0/\sqrt{t+1}$ yields the bound stated in the theorem.

\paragraph{Query Allocation Analysis}
To find the optimal query allocation under a fixed budget $K$ ($\sum q_t \le K$), we analyze the upper bound for large $K$. Let $q_t = q$ be constant, so $T = K/q$. Using asymptotic approximations $\sum_{t=0}^{T-1} (t+1)^{-1} \approx \ln T$ and $\sum_{t=0}^{T-1} (t+1)^{-1/2} \approx 2\sqrt{T}$, the bound becomes roughly:$$U(q) \approx \frac{||x_0 - x^*||^2 + 2\eta_0^2 \sigma^2 \frac{q+d+1}{q} \ln(K/q)}{2\eta_0 \sqrt{K/q}}$$$$U(q) \approx \frac{1}{2\eta_0\sqrt{K}} \left[ \sqrt{q} ||x_0 - x^*||^2 + 2\eta_0^2 \sigma^2 \frac{q+d+1}{\sqrt{q}} \ln(K/q) \right]$$

The first term scales with $\sqrt{q}$, which is strictly increasing in $q$.The second term involves $\frac{q+d+1}{\sqrt{q}} = \sqrt{q} + \frac{d+1}{\sqrt{q}}$. For $q \le d+1$, this term is dominated by $\frac{d+1}{\sqrt{q}}$ (decreasing). The convergence bound is dominated by the first term (the bias term related to the initial distance $||x_0 - x^*||^2$), which scales with $\sqrt{q}$. Since this dominant term is strictly increasing with respect to $q$, the overall upper bound $U(q)$ is minimized when $q$ is chosen to be as small as possible. Therefore, the theoretical analysis suggests that the convergence rate improves with smaller $q$, confirming that the single-query strategy ($q=1$) is the optimal allocation under a fixed query budget. This implies the optimal strategy is $q_t = 1$ for all $t$, setting $T=K$.

\end{proof}

\subsection{Proof of Theorem~\ref{thm:zo_align_stochastic}}
\begin{proof}
\label{Proof:thm_zo_align_sto}
We start by analyzing the evolution of the squared distance to the optimum, $\|x_t - x^*\|^2$. The update rule is $x_{t+1} = x_t - \eta_t \hat{g}_{ALG}(x_t)$.
\begin{equation*}
\|x_{t+1} - x^*\|^2 = \|x_t - x^*\|^2 - 2\eta_t \hat{g}_{ALG}(x_t)^T(x_t - x^*) + \eta_t^2 \|\hat{g}_{ALG}(x_t)\|^2
\end{equation*}
Let $\E_t[\cdot] = \E[\cdot | x_t]$ denote the expectation over the random direction matrix $U_t$ and the stochastic sample $\xi_t$. Using the properties of the ZO-Align estimator, we have $\E_t[\hat{g}_{ALG}(x_t)] = \frac{q_t}{d}\nabla f(x_t)$ and $\E_t[\|\hat{g}_{ALG}(x_t)\|^2] = \frac{q_t}{d}\E_{\xi_t}[\|\nabla F(x_t, \xi_t)\|^2]$.
Using Lemma~\ref{lem:sto_lsmooth}, this becomes:
\begin{equation*}
\E_t[\|\hat{g}_{ALG}(x_t)\|^2] \le \frac{q_t}{d}(4L(f(x_t) - f(x^*)) + 2\sigma^2)
\end{equation*}
Taking the conditional expectation of the distance expansion and using the convexity of $f$, i.e., $f(x_t) - f(x^*) \le \nabla f(x_t)^T(x_t - x^*)$:
\begin{align*}
\E_t[\|x_{t+1} - x^*\|^2] &\le \|x_t - x^*\|^2 - 2\eta_t \frac{q_t}{d} \nabla f(x_t)^T(x_t - x^*) + \eta_t^2 \E_t[\|\hat{g}_{ALG}(x_t)\|^2] \\
&\le \|x_t - x^*\|^2 - \frac{2\eta_t q_t}{d} (f(x_t) - f(x^*)) + \eta_t^2 \frac{q_t}{d}(4L(f(x_t) - f(x^*)) + 2\sigma^2) \\
&= \|x_t - x^*\|^2 + \left(\frac{4L\eta_t^2 q_t}{d} - \frac{2\eta_t q_t}{d}\right)(f(x_t) - f(x^*)) + \frac{2\eta_t^2 q_t}{d}\sigma^2
\end{align*}
Given the step-size condition $\eta_t \le \frac{1}{4L}$, we have $4L\eta_t \le 1$. We can bound the coefficient of the suboptimality gap:
\begin{equation*}
\frac{4L\eta_t^2 q_t}{d} - \frac{2\eta_t q_t}{d} = \frac{\eta_t q_t}{d}(4L\eta_t - 2) \le \frac{\eta_t q_t}{d}(1 - 2) = -\frac{\eta_t q_t}{d}
\end{equation*}
Substituting this bound back into the inequality for the distance:
\begin{equation*}
\E_t[\|x_{t+1} - x^*\|^2] \le \|x_t - x^*\|^2 - \frac{\eta_t q_t}{d}(f(x_t) - f(x^*)) + \frac{2\eta_t^2 q_t}{d}\sigma^2
\end{equation*}
Rearranging, taking the full expectation with $\delta_t = \E[f(x_t) - f(x^*)]$ and $D_t = \E[\|x_t - x^*\|^2]$:
\begin{equation*}
\frac{\eta_t q_t}{d}\delta_t \le D_t - D_{t+1} + \frac{2\eta_t^2 q_t}{d}\sigma^2
\end{equation*}
Summing this inequality from $t=0$ to $T-1$:
\begin{align*}
\sum_{t=0}^{T-1} \frac{\eta_t q_t}{d}\delta_t &\le \sum_{t=0}^{T-1} (D_t - D_{t+1}) + \sum_{t=0}^{T-1} \frac{2\eta_t^2 q_t}{d}\sigma^2 \\
&\le D_0 - D_T + \frac{2\sigma^2}{d} \sum_{t=0}^{T-1} \eta_t^2 q_t \le \|x_0-x^*\|^2 + \frac{2\sigma^2}{d} \sum_{t=0}^{T-1} \eta_t^2 q_t
\end{align*}
Let $\alpha_t = \frac{\eta_t q_t}{d}$. By convexity of $f$ and Jensen's inequality, the suboptimality of the weighted average iterate $\bar{x}_T = \frac{1}{\sum \alpha_t}\sum \alpha_t x_t$ is bounded by $\E[f(\bar{x}_T) - f(x^*)] \le \frac{\sum \alpha_t \delta_t}{\sum \alpha_t}$. This gives:
\begin{equation*}
\E[f(\bar{x}_T) - f(x^*)] \le \frac{\|x_0-x^*\|^2 + \frac{2\sigma^2}{d} \sum_{t=0}^{T-1} \eta_t^2 q_t}{\sum_{t=0}^{T-1} \frac{\eta_t q_t}{d}} = \frac{d\|x_0-x^*\|^2 + 2\sigma^2 \sum_{t=0}^{T-1} \eta_t^2 q_t}{\sum_{t=0}^{T-1} \eta_t q_t}
\end{equation*}
Substituting the step size $\eta_t = \eta_0/\sqrt{t+1}$ yields the bound stated in the theorem.

To optimize the bound, we consider the case where the first term in the numerator (the "bias" term related to the initial distance) dominates the second term (the "variance" term related to $\sigma^2$). This is common in the initial phase of optimization. The goal is to minimize the bound by maximizing the denominator, $\sum_{t=0}^{T-1} \frac{\eta_0 q_t}{\sqrt{t+1}}$, subject to the total budget constraint $\sum_{t=0}^{T-1} q_t \le K$.
Let's analyze the objective $\max \sum_{t=0}^{T-1} \frac{q_t}{\sqrt{t+1}}$. The term $\frac{1}{\sqrt{t+1}}$ represents the "value" or "effectiveness" of queries used at iteration $t$. Since this value is a strictly decreasing function of $t$, we should allocate our query budget to the earliest possible iterations to achieve the maximum sum. This implies a greedy strategy: use the maximum allowed queries, $q_t=d$, for the first $T=K/d$ iterations. This concentrates the entire budget $K$ where it is most effective, thus maximizing the denominator and providing the tightest convergence bound.
\end{proof}

\section{Experiment Setup}
\label{sec:appendix_exp_setup}

\paragraph{Classical objectives.}
All four classical problems use dimension $d = 1000$.

For the strongly convex quadratic, $f(x) = \tfrac{1}{2} x^\top A x + b^\top x$, we first sample a matrix $M \in \mathbb{R}^{d \times d}$ with i.i.d.\ $\mathcal{N}(0,1)$ entries and set
$A = M^\top M + \epsilon I_d$ with a small $\epsilon > 0$ to ensure good conditioning. The vector $b \in \mathbb{R}^d$ is sampled from $\mathcal{N}(0,I_d)$.

For the convex logistic regression objective, we generate synthetic data by sampling feature vectors $a_i \in \mathbb{R}^d$ i.i.d.\ from $\mathcal{N}(0,I_d)$. We draw a ground-truth weight vector $w_{\mathrm{true}} \sim \mathcal{N}(0,I_d)$ and assign binary labels $y_i \in \{-1,1\}$ according to $\mathrm{sign}(a_i^\top w_{\mathrm{true}})$ (with ties broken arbitrarily). The objective is the empirical logistic loss $\tfrac{1}{m}\sum_{i=1}^m \log(1+\exp(-y_i a_i^\top x))$.

The non-convex objective is the standard Rosenbrock function
$f(x) = \sum_{i=1}^{d-1} \bigl[100(x_{i+1}-x_i^2)^2 + (1-x_i)^2\bigr]$.

For the stochastic convex objective, we consider an $\ell_2$-regularized logistic regression loss,
\[
f(x) = \mathbb{E}_{(\mathbf{w}, y) \sim \mathcal{D}}\bigl[\log(1+\exp(-y\,\mathbf{w}^\top x))\bigr] + \tfrac{\rho}{2}\|x\|^2,
\]
where $\mathcal{D}$ is a synthetic distribution generated in the same way as for the convex logistic regression. In the algorithm, this expectation is approximated using mini-batches of examples resampled at each iteration.

\paragraph{Classical experiments: implementation details.}
For each classical objective we run both ZO-Avg and ZO-Align with several query block sizes $q$ and two total query budgets $K \in \{500, 20{,}000\}$, as described in Section~\ref{sec:exp}. Unless otherwise stated, we use the idealized one-sided finite-difference estimator with a fixed smoothing parameter $\mu=10^{-6}$. For ZO-Avg and ZO-Align in the deterministic settings, the step sizes are set according to the theoretically optimal choices in Proposition~\ref{prop:optimal_stepsizes}, namely $\eta_t = \tfrac{q_t}{L(q_t + d + 1)}$ for ZO-Avg and $\eta_t = 1/L$ for ZO-Align, where $L$ is the Lipschitz constant of the gradient (for the quadratic objective, $L$ equals the largest eigenvalue of $A$). In the stochastic convex case we use diminishing step sizes of the form $\eta_t = \eta_0 / \sqrt{t+1}$, with $\eta_0$ tuned by a small grid search. All reported curves are averaged over 10 independent runs with different random seeds.

\paragraph{High-dimensional language model fine-tuning.}
For the high-dimensional experiments we fine-tune the Qwen3-0.6B model on the SST-2 and CB sentence classification datasets using zeroth-order optimization on the cross-entropy loss. We follow a standard sequence-classification protocol: each input sentence (or sentence pair) is fed to the model, the representation of the special classification token is passed through a linear head to produce class logits, and the cross-entropy with respect to the ground-truth labels is used as the objective. All trainable parameters (including the classification head) are flattened into a single vector in $\mathbb{R}^d$, where $d$ is on the order of the total number of parameters of the fine-tuned model.

At each zeroth-order update step we sample $q$ i.i.d.\ Gaussian directions and form the direction matrix $U \in \mathbb{R}^{d \times q}$. ZO-Avg uses the standard averaging estimator based on these $q$ directions. For ZO-Align we use the diagonal approximation described in Section~\ref{sec:exp_high_dim}: instead of explicitly forming and inverting the full Gram matrix $U^\top U$, we compute only its diagonal $D = \mathrm{diag}(U^\top U)$ and apply the update $\hat g_{\mathrm{ALG\text{-}diag}}(x) = U D^{-1} v$, where $v \in \mathbb{R}^q$ collects the one-sided finite-difference directional derivatives. This approximation is motivated by the high-dimensional analysis in Section~\ref{sec:high_dim_approx}, which shows that $U^\top U$ concentrates near a scaled identity when $d \gg q$.

We fix a total query budget of $K = 20{,}000$ function evaluations and consider query block sizes $q \in \{1, 10, 100\}$. For each estimator (ZO-Avg and diagonal ZO-Align) and each choice of $q$, we run the optimizer until the budget $K$ is exhausted. A constant step size is used within each run; for fairness, we tune this step size separately for each (estimator, $q$) pair by a small logarithmic grid search. 

\paragraph{Hardware and Software}
Our experiments were conducted using Python 3.9 and PyTorch 1.12. The underlying system was Ubuntu 18.04, equipped with an Intel Xeon Gold 5320 CPU and two NVIDIA RTX 3090 GPUs.

\end{document}